\documentclass[10pt,twocolumn,letterpaper]{article}

\usepackage{cvpr}              
\definecolor{cvprblue}{rgb}{0.21,0.49,0.74}
\usepackage[breaklinks,colorlinks,allcolors=cvprblue]{hyperref}

\usepackage{algorithm}
\usepackage{algorithmic}
\usepackage{amsmath}
\usepackage{amsthm} 
\usepackage{amsfonts}
\usepackage{booktabs}  
\usepackage{multirow}  
\usepackage{graphicx}  
\usepackage{dsfont}
\newtheorem{theorem}{Theorem}[section]

\newtheorem{proposition}[theorem]{Proposition}

\numberwithin{equation}{section}

\DeclareMathOperator*{\argmin}{arg\,min}
\def\paperID{30599}
\def\confName{CVPR}
\def\confYear{2026}

\title{Conditional Factuality Controlled LLMs with Generalization Certificates via Conformal Sampling}

\author{
Kai Ye \quad Qingtao Pan \quad Shuo Li\thanks{Corresponding author.}\\
Case Western Reserve University\\
{\tt\small Kai.Ye@pitt.edu \quad qingtaopan33@gmail.com \quad shuo.li11@case.edu}
}

\begin{document}
\maketitle
\begin{abstract}
Large language models (LLMs) need reliable test-time control of hallucinations. Existing conformal methods for LLMs typically provide only \emph{marginal} guarantees and rely on a single global threshold, which can under-cover hard prompts, over-cover easy ones, and produce oversized prediction sets. We propose \emph{Conditional Factuality Control} (CFC), a post-hoc conformal framework that returns \emph{set-valued} outputs with \emph{conditional} coverage guarantees. CFC defines a continuous, feature-conditional acceptance threshold through augmented quantile regression on a latent ``success'' score, and deploys it through a fixed-point threshold rule at inference time.
Theoretically, we show that CFC satisfies a conditional coverage guarantee under exchangeability and analyze its \emph{efficiency}, proving that, under mild assumptions on the score distributions, the conditional rule is strictly more sample-efficient than marginal conformal prediction at the same target coverage. We further derive a PAC-style variant, CFC-PAC, which shrinks the nominal risk level based on a stability bound, yielding a finite-sample certificate that the conditional miscoverage deviates from the target by at most $O(\sqrt{\log(1/\delta)/N})$. Empirically, on synthetic data, real-world reasoning and QA benchmarks, and a Flickr8k VLM setting, CFC and CFC-PAC consistently attain near-target coverage across difficulty groups while using smaller prediction sets than CP and non-CP baselines.
\end{abstract}
    
\section{Introduction}
Large language models (LLMs) have delivered striking progress across reasoning and generation tasks~\cite{wei2022emergent,bubeck2023sparks}, yet their outputs can be unreliable due to hallucinations~\cite{huang2025survey}. Inference-time strategies that invest more compute on sampling often improve accuracy~\cite{cobbe2021training,brown2024large}, but they do not provide formal reliability guarantees. For safety‑critical or high‑stakes applications, such heuristics are insufficient: we need procedures that make uncertainty explicit and can control error rates.

\begin{figure}[t]
    \centering
    \includegraphics[width=\linewidth]{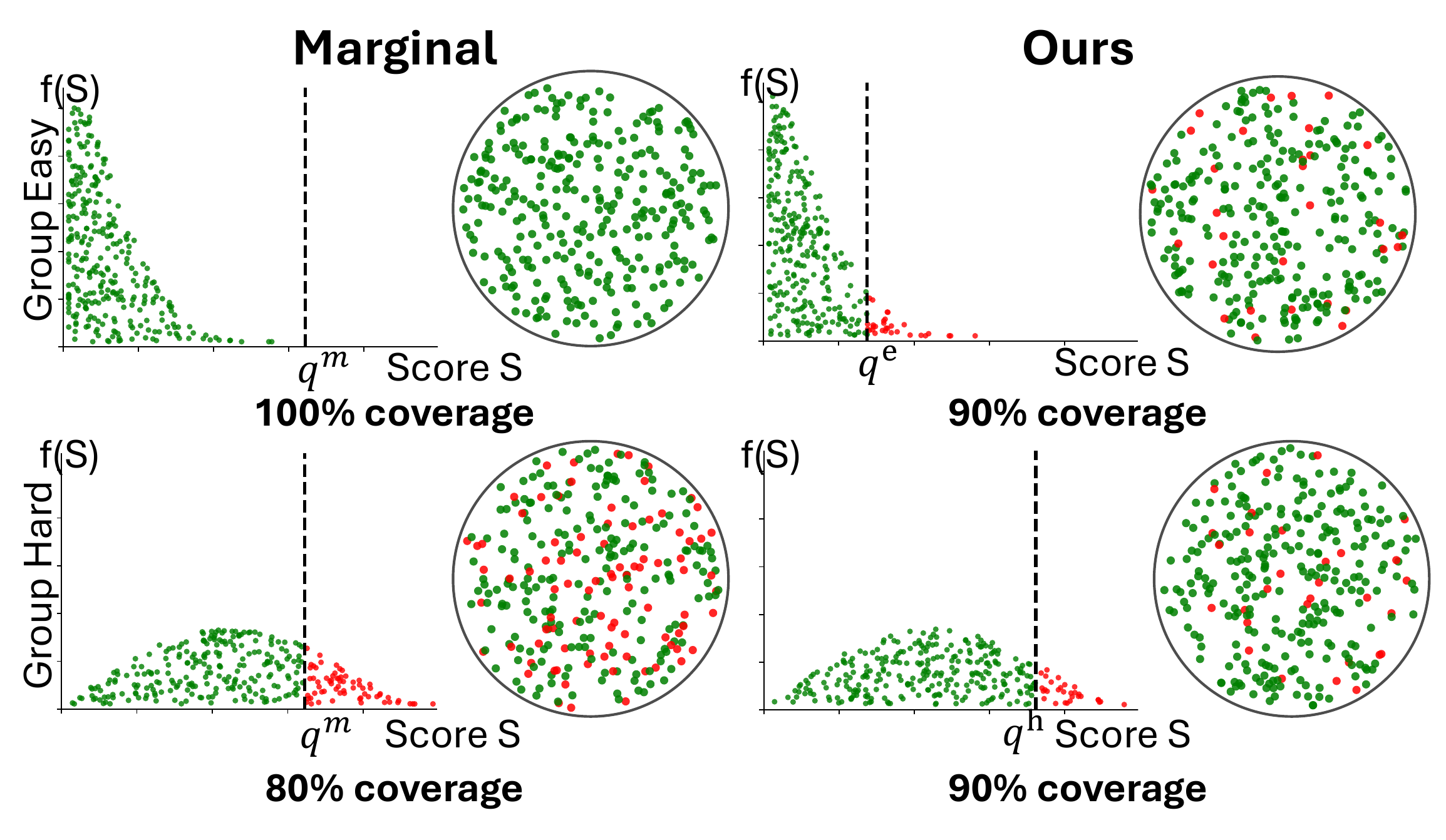}
    \caption{Limitation of marginal CP and advantage of proposed CFC. \textbf{Left:} A single global threshold learned from the marginal score mixture yields only marginal coverage and can under‑cover hard prompts while over‑covering easy ones. \textbf{Right:} Our CFC learns a data‑dependent threshold via conformal quantile regression, adapting the acceptance level to input features and achieving conditional coverage across subgroups.}
    \label{fig:image2}
\end{figure}

Among uncertainty-estimation methods, conformal prediction (CP) is a natural choice for adapting uncertainty control to LLMs, complementing broader efforts on trustworthy uncertainty modeling and multimodal prediction~\cite{vovk2005algorithmic,ye2024uncertainty,ye2025bpen,ye2023bmcl,sensoy2018evidential}. It is model‑agnostic and distribution‑free; under exchangeability between calibration and test examples, it constructs set‑valued predictions that contain the true candidate with probability at least $1-\alpha$ for a user‑specified risk level $\alpha$~\cite{vovk2005algorithmic}. Recent work adapts CP to LLMs by constructing sets of sampled responses that aim to contain at least one correct answer with high probability~\cite{quach2024conformal,kumar2023conformal,su-etal-2024-api}. However, these methods typically rely on a \emph{single global threshold} on scores, and so they only provide \emph{marginal} guarantees: coverage holds on average over prompts, not for prompts with particular characteristics.

This marginal coverage can hide severe heterogeneity. Hard prompts (e.g., long math questions or rare entities) may be systematically under‑covered, while easy prompts are over‑covered, potentially inflating prediction sets unnecessarily (Fig.~\ref{fig:image2}, left). A global threshold is forced to compromise between easy and hard regions of the feature space, leading to miscalibration in subgroups and inefficient use of samples.


\textit{Motivation.}
To address this failure mode, we seek \emph{conditional coverage}: guarantees that coverage holds not only on average, but also when conditioning on relevant features (or groups) of the prompt. Conditional coverage is strictly stronger than marginal coverage and directly targets reliability in under‑represented or systematically hard subpopulations (Fig.~\ref{fig:image2}, right). At the same time, we would like prediction sets to remain compact on average so that sampling‑based inference stays computationally practical.

We tackle these challenges with \emph{Conditional Factuality Control} (CFC), a post-hoc conformal layer for LLM sampling. Rather than using a single scalar threshold, CFC defines a \emph{feature-conditional acceptance rule} $\widehat{\lambda}_\alpha(X)$ by conformalizing a quantile-regression model for a latent success score $S(X)$, namely the best score among correct candidates for prompt $X$. At test time, given a prompt $X$, we sample candidates from the base generator and accept those whose scores satisfy $V(X,y)\le \widehat{\lambda}_\alpha(X)$. This requires no finetuning of the base model, and lets the acceptance threshold adapt to prompt difficulty.

Beyond the basic procedure, we develop a \emph{PAC-style} variant, \emph{CFC-PAC}, that adds a stability-based finite-sample certificate: with high probability over the draw of calibration sample, the deployed rule achieves coverage at least $1-\alpha-\varepsilon_N(\delta)$, where $\varepsilon_N(\delta)=O\!\big(\sqrt{\log(1/\delta)/N}\big)$.

Finally, we study the \emph{efficiency} of CFC. Under natural assumptions relating prompt difficulty to the distribution of scores, we show that an oracle conditional rule can attain smaller expected prediction-set size than marginal CP rule at the same target coverage. Our learned CFC asymptotically inherits this oracle efficiency as the augmented quantile regression becomes consistent.

\noindent
Our main contributions are:
\begin{itemize}
    \item We introduce \textbf{CFC}, a conformal procedure for sampled LLM outputs that defines a continuous, feature-conditional acceptance rule via augmented quantile regression on the latent success score and yields prediction sets satisfying the conditional guarantee of~\cite{gibbs2024conformal} under exchangeability.
    \item We develop \textbf{CFC-PAC}, a certified variant of CFC equipped with a PAC-style generalization bound: with probability at least $1-\delta$ over the calibration sample, the deployed rule achieves coverage at least $1-\alpha-\varepsilon_N(\delta)$ for an $\varepsilon_N(\delta)=O\!\big(\sqrt{\tfrac{\log(1/\delta)}{N}}\big)$.
    \item We analyze the \textbf{efficiency} of CFC, proving that, under mild monotonicity and concavity assumptions on the score distribution, conditional rules can be strictly more sample-efficient than marginal CP rules at the same coverage level.
    \item We validate CFC and CFC-PAC on synthetic data, real-world reasoning / QA benchmarks, and a Flickr8k VLM setting, showing that the same post-hoc procedure extends beyond text-only generators without finetuning the base model.
\end{itemize}

\noindent\textbf{Code availability.} Code for reproducing our experiments is available \href{https://github.com/FlynnYe/CFC-LLMs}{here}.

\section{Related Work}

\paragraph{Inference-time sampling and reranking for LLMs.}
A common way to improve LLM outputs with extra test-time compute is to sample multiple candidates and rerank or filter them, as in Best-of-$N$ decoding and pass@$N$ evaluation~\cite{cobbe2021training,chen2021evaluating,brown2024large}. In practice, candidate quality is often estimated with an external verifier or reward model~\cite{cobbe2021training,hosseini2024v}. Our setting follows this line of work: we treat the base generator as a black box, sample multiple candidates, and use a verifier score to decide which candidates enter the final set.

\paragraph{Conformal prediction and conditional guarantees.}
Conformal prediction provides distribution-free finite-sample coverage under exchangeability~\cite{vovk2005algorithmic,shafer2008tutorial}. Split/inductive conformal prediction (ICP) is the standard practical variant, but its guarantee is marginal: it controls coverage only on average over test inputs~\cite{papadopoulos2002inductive}. Exact conditional coverage is impossible without additional assumptions or relaxations~\cite{vovk2012conditional,foygel2021limits}. Recent work therefore studies weaker but useful coverages, including conformalized quantile regression~\cite{romano2019conformalized} and the function-class conditional framework of Gibbs et al.~\cite{gibbs2024conformal}, which learns feature-dependent thresholds through augmented quantile regression. Our method builds on this latter perspective.

\paragraph{Conformal prediction for LLMs.}
Several recent papers adapt CP to language models by constructing sets of sampled responses that contain at least one correct answer with high probability~\cite{kumar2023conformal,quach2024conformal,su-etal-2024-api,cherian2024large,mohri2024language}. Most existing LLM-specific methods use a \emph{global} acceptance rule and therefore inherit only marginal coverage, which can under-cover hard prompts and over-cover easy ones. Relative to these marginal baselines, CFC replaces the single threshold with a feature-conditional one and is therefore designed to improve subgroup reliability. Compared with prior conditional CP works, our contributions are different in emphasis: we develop conformal factuality control for sampled LLM candidates with verifier scores, provide an efficiency analysis showing when conditional rules are more sample-efficient than marginal ones, introduce the PAC-certified variant CFC-PAC, and demonstrate transfer to a VLM setting in the main experiments. Because CFC is purely post hoc, it requires no generator finetuning and transfers across base models.

\section{Preliminaries}
\label{sec:setup}

\subsection{Conformal Factuality}
Let $X\in\mathcal X$ be a prompt and let
$\pi:\mathcal X\to\Delta(\mathcal Y)$ denote a fixed generator over candidate completions.
For each prompt, we sample a candidate set
$C(X)=\{Y_j\}_{j=1}^M$ with $Y_j\sim\pi(\cdot\mid X)$,
and evaluate each candidate with a verifier score
$V:\mathcal X\times\mathcal Y\to[0,1]$, normalized so that \emph{smaller scores are better}.
Given a correctness indicator $A(X,y)\in\{0,1\}$, the conformal factuality goal is to output a set
$\widehat C_\alpha(X)$ satisfying
\begin{equation}
\mathbb{P}\!\left(\exists y \in \widehat C_\alpha(X_{n+1}) : A(X_{n+1},y)=1\right) \ge 1-\alpha.
\end{equation}
A convenient latent variable for this event is the \emph{success score}
\begin{equation}
S(X)
:=
\inf\{\,V(X,y): y\in C(X),\ A(X,y)=1\,\},
\end{equation}
so that the event that the prediction set contains at least one correct answer is equivalent to
$S(X)\le \lambda(X)$ for the deployed threshold rule $\lambda(\cdot)$.

\subsection{From Marginal to Conditional Coverage}
\paragraph{Split conformal prediction.}
Standard split conformal prediction calibrates a single global threshold from held-out calibration scores and therefore provides only \emph{marginal} coverage. Given exchangeable data $\{(X_i,Y_i)\}_{i=1}^{n+1}$ and a nonconformity score $s(x,y)$, it constructs a prediction set $\widehat C_\alpha(X_{n+1})$ satisfying
\begin{equation}
\mathbb P\!\left(Y_{n+1}\in \widehat C_\alpha(X_{n+1})\right)\ge 1-\alpha.
\end{equation}
Using a calibration set $\{(X_i,Y_i)\}_{i=1}^n$, one computes calibration scores $R_i=s(X_i,Y_i)$ and forms the empirical quantile
\begin{equation}
\widehat q_{1-\alpha}:=R_{(\lceil (n+1)(1-\alpha)\rceil)}
\end{equation}
and predicts
\begin{equation}
\widehat C_\alpha(x):=\{y:s(x,y)\le \widehat q_{1-\alpha}\}.
\end{equation}
In our setting, this corresponds to accepting all sampled candidates with verifier score below one constant cutoff. Such a rule can over-cover easy prompts and under-cover hard prompts, because the same acceptance level must serve the entire prompt distribution.

\paragraph{Conditional conformal prediction.}
The limitation of marginal CP is that it averages over the prompt distribution. A stronger objective is pointwise conditional coverage,
\begin{equation}
\mathbb P\!\left(Y_{n+1}\in \widehat C(X_{n+1}) \mid X_{n+1}=x\right)\ge 1-\alpha,
\end{equation}
but exact finite-sample conditional coverage is impossible without additional assumptions or relaxations~\cite{vovk2012conditional,foygel2021limits}. Following Gibbs et al.~\cite{gibbs2024conformal}, we instead target function-class conditional coverage:
\begin{equation}
\begin{aligned}
\mathbb{E}\!\Big[&f(X_{n+1})\big(\mathds{1}\{\exists y\in \widehat C(X_{n+1}) : A(X_{n+1},y)=1\}\\
&-(1-\alpha)\big)\Big]=0
\end{aligned}
\end{equation}
for all $f$ in a chosen class $\mathcal F$.
When $\mathcal F=\{1\}$, this reduces to marginal conformal prediction; when
$\mathcal F=\{\Phi(\cdot)^\top\beta : \beta\in\mathbb{R}^d\}$,
it yields a feature-conditional guarantee over the basis $\Phi$.

CFC instantiates this framework for conformal factuality by learning a feature-dependent quantile of the latent success score $S(X)$ rather than a single global cutoff.
The main paper focuses on the procedure and its guarantees. Appendix~\ref{app:background} gives the detailed background about ICP and Gibbs.~\cite{gibbs2024conformal} work.

\section{Method}
\label{sec:method}

We now present \textbf{CFC}, our conditional conformal rule for sampled LLM outputs, and its PAC-style variant \textbf{CFC-PAC}. Appendices~\ref{app:background} and \ref{app:proofs} contain the extended background, derivations, and proofs.

\subsection{CFC: Conditional Factuality Control}
\label{subsec:cfc}

Let $\{(X_i,S_i)\}_{i=1}^N$ be calibration prompt/score pairs, where $S_i=S(X_i)$ is the latent success score from Sec.~\ref{sec:setup}. For a test prompt $X_{N+1}$, we sample $C_{N+1}=\{Y_{N+1,j}\}_{j=1}^M$ from $\pi(\cdot\mid X_{N+1})$ and deploy a prompt-dependent threshold on verifier scores:
\begin{equation}
\widehat C_{\alpha}(X_{N+1};\lambda)
:=
\{y\in C_{N+1}: V(X_{N+1},y)\le \lambda\}.
\end{equation}

To learn this threshold, we follow the augmented quantile-regression construction of~\cite{gibbs2024conformal}. For a candidate test-time score $s\in[0,1]$, define
\begin{equation}
\label{eq:aug-qr-obj}
\begin{aligned}
\beta_s
=
\argmin_{\beta\in\mathbb{R}^d}\Bigg[
\frac{1}{N+1}\sum_{i=1}^{N}\rho_{1-\alpha}\!\left(S_i-\Phi(X_i)^\top\beta\right)
\\
+\frac{1}{N+1}\rho_{1-\alpha}\!\left(s-\Phi(X_{N+1})^\top\beta\right)\Bigg],
\end{aligned}
\end{equation}
where $\Phi(X)$ is the chosen feature map and $\rho_{1-\alpha}(u)=u(1-\alpha-\mathds 1\{u<0\})$ is the pinball loss. The induced test-time map is
\[
g_{X_{N+1}}(s):=\Phi(X_{N+1})^\top\beta_s.
\]
We then take the largest fixed point below this map as the deployed threshold:
\begin{equation}
\label{eq:pre_set}
\widehat\lambda_\alpha(X_{N+1})
:=
\sup\{s\in[0,1]: s\le g_{X_{N+1}}(s)\},
\end{equation}
and return
\begin{equation}
\widehat C_\alpha(X_{N+1})
:=
\{y\in C_{N+1}: V(X_{N+1},y)\le \widehat\lambda_\alpha(X_{N+1})\}.
\end{equation}
Although Eqs.~\eqref{eq:aug-qr-obj}--\eqref{eq:pre_set} define CFC conceptually through an augmented quantile-regression family, deployment only requires computing the fixed-point threshold in Eq.~\eqref{eq:pre_set}; it does not require repeatedly solving Eq.~\eqref{eq:aug-qr-obj} on a grid of candidate scores (see Gibbs.~\cite{gibbs2024conformal}).

\begin{algorithm}[t]
\small
\caption{CFC Inference}
\label{alg:infer}
\begin{algorithmic}[1]
\REQUIRE Calibration pairs $\{(X_i,S_i)\}_{i=1}^{N}$, test prompt $x$, generator $\pi$, verifier $V$, feature map $\Phi$, nominal significance $\alpha$, sample budget $M$
\STATE Draw $C(x)=\{Y_j\}_{j=1}^M$ with $Y_j\sim\pi(\cdot\mid x)$
\STATE Let $\rho_{1-\alpha}(u)=u(1-\alpha-\mathds 1\{u<0\})$
\STATE For $s\in[0,1]$, define
\STATE $\displaystyle
\beta_s=\argmin_{\beta\in\mathbb R^d}\Bigg[\frac{1}{N+1}\sum_{i=1}^N\rho_{1-\alpha}\!\left(S_i-\Phi(X_i)^\top\beta\right)+\frac{1}{N+1}\rho_{1-\alpha}\!\left(s-\Phi(x)^\top\beta\right)\Bigg]$
\STATE Define $g_x(s)=\Phi(x)^\top\beta_s$
\STATE Compute $\widehat\lambda_\alpha(x)=\sup\{s\in[0,1]: s\le g_x(s)\}$
\STATE Set $\widehat C_\alpha(x)=\{y\in C(x): V(x,y)\le \widehat\lambda_\alpha(x)\}$
\RETURN $\widehat\lambda_\alpha(x),\widehat C_\alpha(x)$
\end{algorithmic}
\end{algorithm}

\begin{theorem}[Conditional coverage of \textsc{CFC}]
\label{thm:cfc_cov}
Let $\mathcal{F}=\{\Phi(X)^\top\beta: \beta\in\mathbb{R}^d\}$ be any finite-dimensional linear class, and assume exchangeability. Then for any non-negative $f\in\mathcal F$ with $\mathbb E[f(X)]>0$, the prediction set in Eq.~\eqref{eq:pre_set} satisfies
\begin{equation}
\mathbb P_f\!\left(\exists y\in\widehat C_\alpha(X_{N+1}) : A(X_{N+1},y)=1\right)\ge 1-\alpha.
\end{equation}
\end{theorem}


\subsection{CFC-PAC: A High-Probability Certificate}
\label{subsec:pac}

Theorem~\ref{thm:cfc_cov} is an expectation-level guarantee over the calibration draw. To certify the deployed rule itself, we add ridge regularization to Eq.~\eqref{eq:aug-qr-obj} and shrink the nominal target by a stability slack.

\paragraph{Assumptions.}
We assume: (1) bounded features $\|\Phi(X)\|_2\le R$ almost surely; (2) ridge regularization $\frac{\lambda}{2}\|\beta\|_2^2$ with $\lambda>0$ in the augmented quantile-regression objective; and (3) the conditional CDF $F_{S\mid X=x}(t)$ is $L$-Lipschitz in $t$ on $[0,1]$. Exchangeability is as in Theorem~\ref{thm:cfc_cov}.

\begin{theorem}[PAC conditional coverage for CFC]
\label{thm:pac-cfc}
Assume $\alpha \ge \varepsilon_N(\delta)$ and define
$
\alpha_{\mathrm{eff}} := \alpha - \varepsilon_N(\delta).
$ Let $\widehat\lambda_{\alpha_{\mathrm{eff}}}(\cdot)$ be the threshold learned from algorithm~\ref{alg:infer-pac}. Then for any $\delta\in(0,1)$, with probability at least $1-\delta$ over the calibration sample,
\[
\mathbb P\!\left(S\le \widehat\lambda_{\alpha_{\mathrm{eff}}}(X)\mid \mathcal D_{\mathrm{cal}}\right)
\ge
1-{\alpha_{\mathrm{eff}}}-\varepsilon_N(\delta)
=
1-\alpha,
\]
where
\[
\varepsilon_N(\delta)=O\!\left(\sqrt{\frac{\log(1/\delta)}{N}}\right).
\]
Equivalently, with the same probability,
\[
\mathbb P\!\left(\exists y\in\widehat C_{\alpha_{\mathrm{eff}}}(X): A(X,y)=1\mid \mathcal D_{\mathrm{cal}}\right)
\ge
1-\alpha.
\]
\end{theorem}


\begin{algorithm}[t]
\small
\caption{CFC-PAC Inference}
\label{alg:infer-pac}
\begin{algorithmic}[1]
\REQUIRE Calibration pairs $\{(X_i,S_i)\}_{i=1}^{N}$, test prompt $x$, generator $\pi$, verifier $V$, feature map $\Phi$, nominal significance $\alpha$, confidence $\delta$, ridge parameter $\lambda$, sample budget $M$
\STATE Compute the PAC slack $\varepsilon_N(\delta)$ from the PAC certificate
\STATE Set the effective target $\alpha_{\mathrm{eff}}=\max\{0,\alpha-\varepsilon_N(\delta)\}$
\STATE Similar to Algorithm~\ref{alg:infer}, run CFC with target $\alpha_{\mathrm{eff}}$, replacing the augmented objective by the ridge-regularized version below
\STATE $\displaystyle
\beta_s=\argmin_{\beta\in\mathbb R^d}\Bigg[\frac{1}{N+1}\sum_{i=1}^N\rho_{1-\alpha_{\mathrm{eff}}}\!\left(S_i-\Phi(X_i)^\top\beta\right)+\frac{1}{N+1}\rho_{1-\alpha_{\mathrm{eff}}}\!\left(s-\Phi(x)^\top\beta\right)+\frac{\lambda}{2}\|\beta\|_2^2\Bigg]$
\STATE This produces $\widehat\lambda_{\alpha_{\mathrm{eff}}}(x)$ and $\widehat C_{\alpha_{\mathrm{eff}}}(x)$
\RETURN $\widehat\lambda_{\alpha_{\mathrm{eff}}}(x),\widehat C_{\alpha_{\mathrm{eff}}}(x)$
\end{algorithmic}
\end{algorithm}

\subsection{Efficiency Analysis}
\label{subsec:efficiency}

Beyond coverage, we ask whether conditioning can reduce average prediction-set size. Let
\[
G_X(\lambda):=\mathbb P\big(V(X,Y)\le \lambda\mid X,\ Y\sim\pi(\cdot\mid X)\big)
\]
be the score CDF under sampling at prompt $X$. If we draw $M$ candidates and retain those with score at most $\lambda$, then
\[
\mathbb E\big[|\widehat C(X)|\mid X\big]=M\,G_X(\lambda).
\]
Thus efficiency reduces to comparing $\mathbb E[G_X(\lambda(X))]$ across threshold rules.

Marginal CP uses a constant threshold $\bar\lambda_\alpha$ satisfying $\mathbb P(S\le \bar\lambda_\alpha)\approx 1-\alpha$. In contrast, an oracle conditional rule uses the conditional quantile $q_\alpha(X)$ of $S\mid X$ and sets $\lambda^\star(X)=q_\alpha(X)$.
Let $T=\psi(X)$ be a scalar difficulty variable. For each $t$, define
$
F_t(\lambda):=\mathbb P(S\le \lambda\mid T=t),
\qquad
G_t(\lambda):=\mathbb P(V\le \lambda\mid T=t),
$
and
$
C_t(u):=G_t\!\big(F_t^{-1}(u)\big),
\qquad
u\in[0,1],
$
where
$
F_t^{-1}(u):=\inf\{\lambda\in[0,1]:F_t(\lambda)\ge u\}.
$
\paragraph{Assumptions.}
(1) For each $t$, $F_t$ is continuous and strictly increasing on $[0,1]$.
(2) For each fixed $\lambda\in[0,1]$, the map $t\mapsto F_t(\lambda)$ is nonincreasing.
(3) For each $t$, the map $u\mapsto C_t(u)$ is convex and differentiable on $(0,1)$.
(4) $C$ has decreasing differences in $(u,t)$, i.e. for every $u\in(0,1)$,$t\mapsto \partial_u C_t(u)$ is nonincreasing.
Under Assumptions 1--2, the conditional $(1-\alpha)$-quantile
$q_\alpha(t):=F_t^{-1}(1-\alpha)$
is automatically nondecreasing in $t$.

\begin{proposition}[Oracle CFC efficiency]
\label{prop:oracle-eff}
Let
\[
\lambda^\star(X):=q_\alpha(X)=F_X^{-1}(1-\alpha).
\]
Under the assumptions above,
\[
\mathbb E\big[G_X(\lambda^\star(X))\big]
\le
\mathbb E\big[G_X(\bar\lambda)\big]
\]
for any constant $\bar\lambda$ satisfying
\[
\mathbb P(S\le \bar\lambda)=1-\alpha.
\]
In particular,
\[
\mathbb E\big[G_X(\lambda^\star(X))\big]
\le
\mathbb E\big[G_X(\bar\lambda_\alpha)\big].
\]

If, in addition, $u\mapsto C_t(u)$ is strictly convex for almost every $t$ and
\[
\mathbb P\big(q_\alpha(X)\neq \bar\lambda_\alpha\big)>0,
\]
then the inequality is strict.
\end{proposition}


\begin{theorem}[CFC inherits oracle efficiency]
\label{thm:cfc-eff}
Assume the conditions of Proposition~\ref{prop:oracle-eff}, and let
\[
\lambda^\star(X):=q_\alpha(\psi(X)).
\]
Let $\widehat\lambda_{\alpha,N}(\cdot)$ denote the CFC threshold learned from a
calibration sample of size $N$, and suppose that
\[
\sup_{x\in\mathcal X}
\big|\widehat\lambda_{\alpha,N}(x)-\lambda^\star(x)\big|
\xrightarrow{p}0
\qquad\text{as }N\to\infty.
\]
Then
\[
\lim_{N\to\infty}\mathbb E\big[G_X(\widehat\lambda_{\alpha,N}(X))\big]
=
\mathbb E\big[G_X(\lambda^\star(X))\big]
\le
\mathbb E\big[G_X(\bar\lambda_\alpha)\big].
\]
Consequently, for fixed $M$,
\[
\lim_{N\to\infty}\mathbb E\big[|\widehat C_{\alpha,N}(X)|\big]
=
M\,\mathbb E\big[G_X(\lambda^\star(X))\big]
\le
M\,\mathbb E\big[G_X(\bar\lambda_\alpha)\big].
\]

If, in addition, the strictness conditions in
Proposition~\ref{prop:oracle-eff} hold, then
\[
\lim_{N\to\infty}\mathbb E\big[G_X(\widehat\lambda_{\alpha,N}(X))\big]
<
\mathbb E\big[G_X(\bar\lambda_\alpha)\big],
\]
and hence
\[
\lim_{N\to\infty}\mathbb E\big[|\widehat C_{\alpha,N}(X)|\big]
<
M\,\mathbb E\big[G_X(\bar\lambda_\alpha)\big].
\]
\end{theorem}

\section{Experiments}
\label{sec:experiments}

\subsection{Synthetic Data}
\label{sec:synthetic}

\paragraph{Setup.}
We first study a controlled synthetic setting with a scalar prompt-difficulty variable $T\in[0,1]$. For each prompt we draw $M$ candidates, the probability of correctness decreases with difficulty, and verifier scores are sampled from a known difficulty-dependent distribution. Unless otherwise stated, we use target error $\alpha=0.10$, calibration size $N_{\mathrm{cal}}=10{,}000$, test size $N_{\mathrm{test}}=10{,}000$, and $M=50$. We report mean $\pm$ standard deviation over random seeds.

\paragraph{Baselines and metrics.}
We compare against three baselines. \textbf{TopK} keeps the smallest top-$K$ prefix whose calibration coverage reaches the target level. \textbf{ICP}~\cite{vovk2005algorithmic} calibrates a single global verifier threshold. \textbf{Learnt CP} fits a feature-conditional threshold from calibration data but omits the exact conformal correction. We report \textbf{empirical coverage rate} (ECR), \textbf{average prediction set size} (APSS, lower is better), and \textbf{group-stratified coverage} (GSC), the minimum empirical coverage over difficulty groups. On synthetic data, ECR and GSC are computed using the ground-truth correctness labels, and APSS is the average accepted set size.

\begin{figure}[t]
  \centering
  \includegraphics[width=0.80\linewidth]{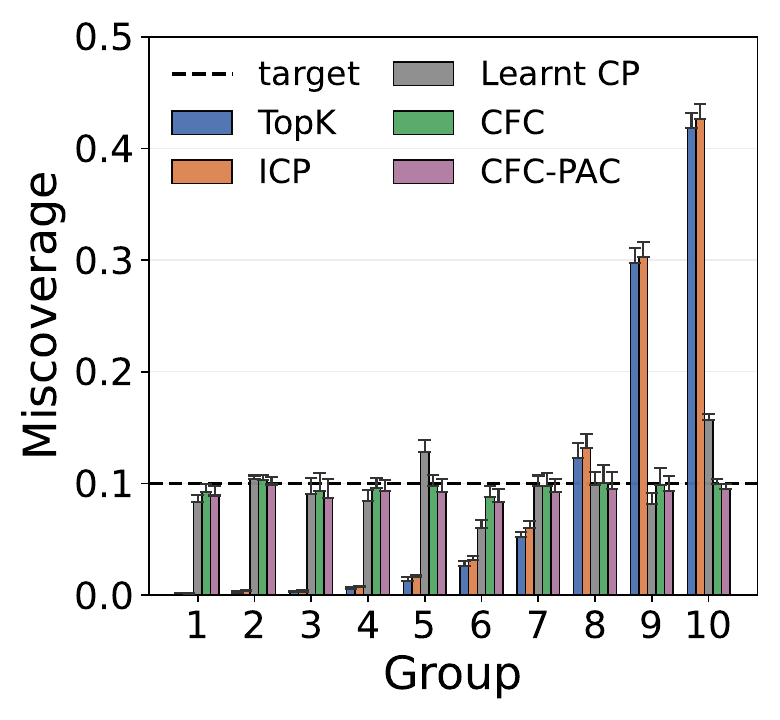}
  \caption{Groupwise miscoverage on synthetic data across 10 difficulty bins. The dashed line marks the target miscoverage $\alpha = 0.10$. {Learnt CP} improves over marginal baselines, but CFC and CFC-P remain closest to the target across all bins, especially on hard prompts.}
  \label{fig:synfg1}
\end{figure}

\paragraph{Results.}
Figure~\ref{fig:synfg1} shows the main phenomenon motivating CFC: a single global threshold under-covers hard prompts and over-covers easy ones, while conditional thresholds track the target miscoverage across the entire difficulty range. {Learnt CP} narrows this gap, but it still fails to match the subgroup reliability of CFC/CFC-P on the hardest bins. CFC achieves the smallest average prediction sets at the target level, and CFC-P raises the worst-group coverage floor further with only a modest set-size increase. This shows that the gain does not come merely from learning a better global score; it comes from feature-conditional conformalization.

\begin{figure}[t]
  \centering
  \includegraphics[width=0.80\linewidth]{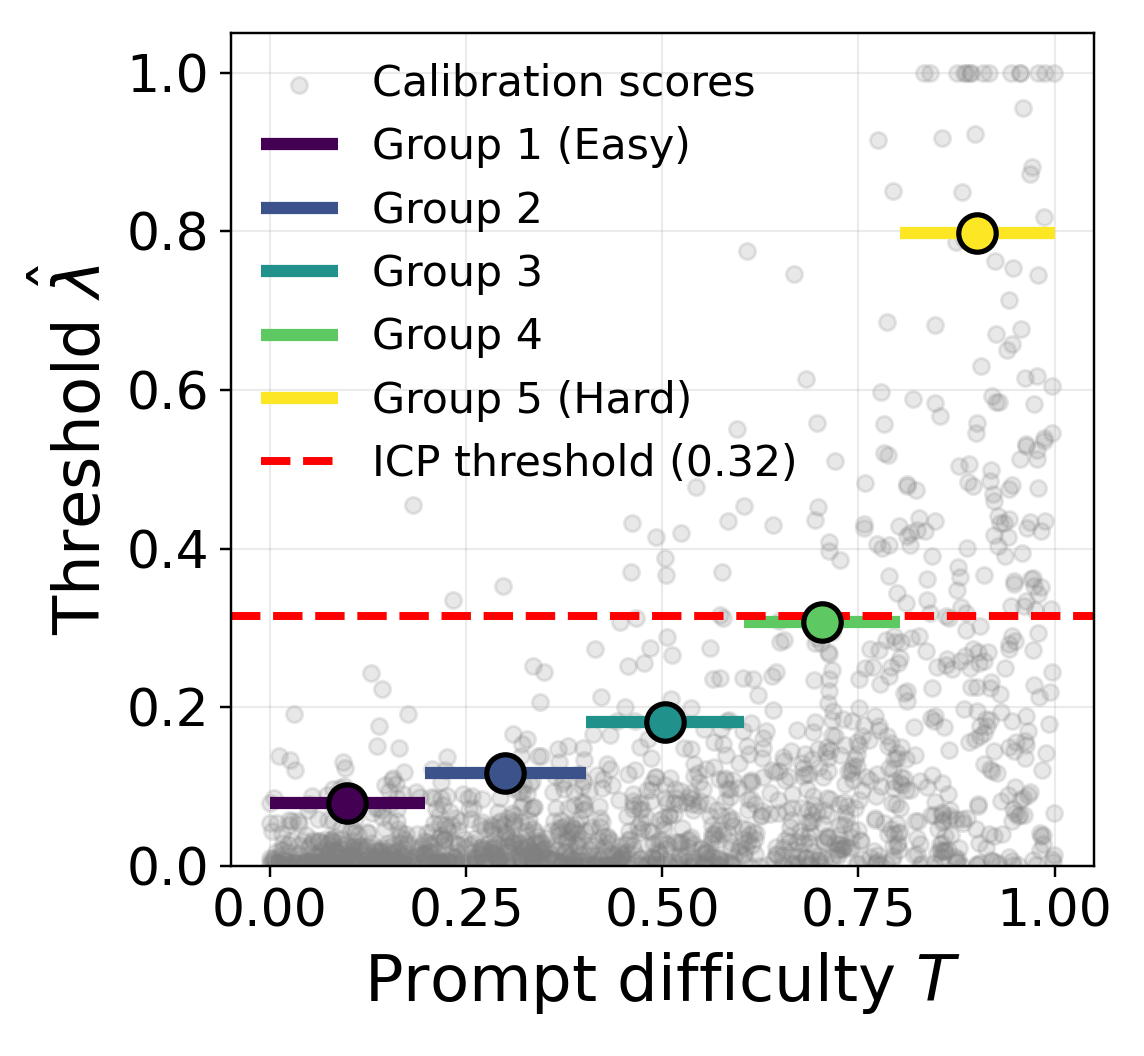}
  \caption{Learned threshold $\widehat{\lambda}_\alpha(X)$ versus prompt difficulty. Easy prompts receive stricter thresholds, while harder prompts receive looser thresholds, which is the mechanism behind the improved group-wise reliability of CFC.}
  \label{fig:lambda_vs_difficulty_main}
\end{figure}

\paragraph{Threshold adaptation.}
Figure~\ref{fig:lambda_vs_difficulty_main} visualizes the learned threshold itself. Relative to the single ICP cutoff, CFC assigns tighter thresholds to easy prompts and looser thresholds to hard prompts, which is exactly the behavior needed to correct the systematic under-coverage of global-threshold baselines on difficult inputs.

\begin{table}[t]
\centering
\small
\caption{Synthetic-data results at $\alpha=0.10$. {Learnt CP} uses a learned difficulty-conditional threshold but no exact conformal correction, isolating learning from feature-conditional conformalization.}
\label{tab:synth_cvpr_compact_main}
\setlength{\tabcolsep}{3pt}
\begin{tabular}{lccc}
\toprule
Method & ECR & APSS$\downarrow$ & GSC$\uparrow$ \\
\midrule
TopK & 90.6 $\pm$ 0.1 & 16.00 $\pm$ 0.00 & 58.2 $\pm$ 1.3 \\
ICP & 90.2 $\pm$ 0.2 & 16.71 $\pm$ 0.23 & 57.4 $\pm$ 1.4 \\
Learnt CP & 90.2 $\pm$ 0.3 & 15.72 $\pm$ 0.15 & 84.3 $\pm$ 0.5 \\
\textbf{CFC (ours)} & 90.3 $\pm$ 0.5 & \textbf{15.53 $\pm$ 0.12} & 88.7 $\pm$ 0.7 \\
\textbf{CFC-P (ours)} & 90.8 $\pm$ 0.6 & 15.87 $\pm$ 0.16 & \textbf{89.1 $\pm$ 0.5} \\
\bottomrule
\end{tabular}
\setlength{\tabcolsep}{6pt}
\end{table}

Table~\ref{tab:synth_cvpr_compact_main} shows that learning a stronger difficulty-aware threshold already helps substantially, but it still fails to match the subgroup reliability of CFC/CFC-P. This isolates the main empirical point of the method: the gains are not explained by better score fitting alone, but by the exact conditional conformal correction on top of that learned threshold.

\paragraph{Additional Sweeps and Ablations.}
Appendix~\ref{app:exp-details} reports the complete synthetic sweep across target error rates $\alpha$ together with additional sensitivity analyses over calibration size, sampling budget, and group granularity.

\begin{figure*}[t]
  \centering
  \includegraphics[width=\textwidth]{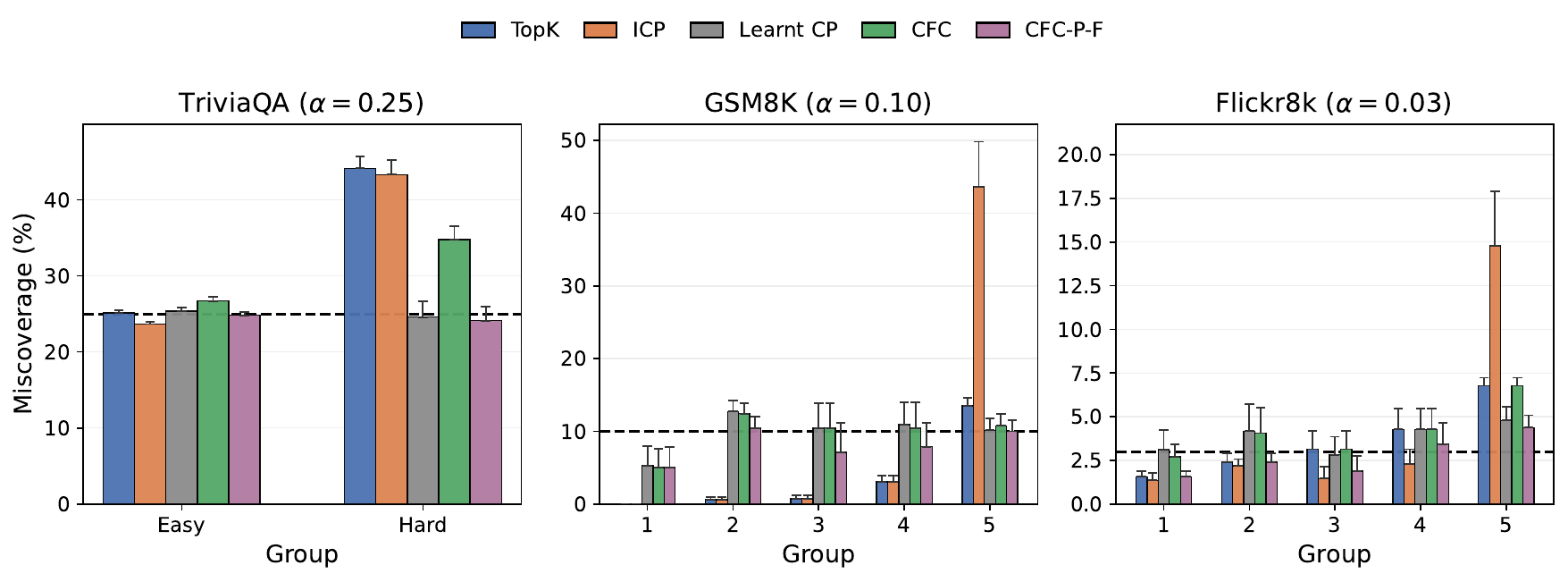}
  \caption{Groupwise miscoverage on real-world datasets at representative target errors: TriviaQA $\alpha=0.25$, GSM8K $\alpha=0.10$, and Flickr8k $\alpha=0.03$. Bars show mean miscoverage over split seeds and the upper error bars show one standard deviation. The TriviaQA panel uses the chosen two-group feature map; Appendix~\ref{app:exp-details} gives its exact construction. The GSM8K and Flickr8k panels use five equal-frequency difficulty groups ordered from easy to hard. The dashed line marks the target miscoverage $\alpha$. Across all three datasets, the conditional methods flatten the miscoverage profile relative to marginal baselines, especially on the hardest groups.}
  \label{fig:real_group_miscoverage}
\end{figure*}

\begin{table*}[!t]
\centering
\small
\caption{TriviaQA results for $\alpha\in\{0.20,0.25,0.30,0.35\}$ under the chosen calibration-defined feature map. APSS below $1$ indicates abstention on some prompts. ECR and GSC are reported in percent; for ECR, values closest to the target coverage $1-\alpha$ are preferred. Appendix~\ref{app:exp-details} reports the full experiments with additional details.}
\label{tab:trivia_compact}
\setlength{\tabcolsep}{3.5pt}
\begin{tabular}{l ccc ccc}
\toprule
Methods &
\multicolumn{3}{c}{$\alpha=0.20$} &
\multicolumn{3}{c}{$\alpha=0.25$} \\
\cmidrule(lr){2-4}\cmidrule(lr){5-7}
& ECR & GSC$\uparrow$ & APSS$\downarrow$
& ECR & GSC$\uparrow$ & APSS$\downarrow$ \\
\midrule
\textbf{TopK} &
81.6 $\pm$ 0.3 & 69.7 $\pm$ 1.5 & 1.75 $\pm$ 0.00 &
73.4 $\pm$ 0.3 & 55.9 $\pm$ 1.6 & \textbf{1.00 $\pm$ 0.00} \\
\textbf{ICP} &
80.0 $\pm$ 0.5 & 65.1 $\pm$ 1.5 & 1.43 $\pm$ 0.02 &
74.9 $\pm$ 0.3 & 56.7 $\pm$ 1.9 & 1.08 $\pm$ 0.01 \\
\textbf{Learnt CP} &
79.6 $\pm$ 0.5 & 76.3 $\pm$ 1.5 & 1.69 $\pm$ 0.04 &
74.7 $\pm$ 0.4 & 74.0 $\pm$ 1.1 & 1.22 $\pm$ 0.03 \\
\textbf{CFC (ours)} &
76.2 $\pm$ 0.5 & 65.4 $\pm$ 1.7 & \textbf{1.21 $\pm$ 0.01} &
72.7 $\pm$ 0.4 & 65.2 $\pm$ 1.8 & 1.03 $\pm$ 0.03 \\
\textbf{CFC-P-F (ours)} &
80.1 $\pm$ 0.5 & \textbf{76.3 $\pm$ 1.6} & 1.72 $\pm$ 0.04 &
75.3 $\pm$ 0.4 & \textbf{74.6 $\pm$ 1.0} & 1.32 $\pm$ 0.10 \\
\midrule
Methods &
\multicolumn{3}{c}{$\alpha=0.30$} &
\multicolumn{3}{c}{$\alpha=0.35$} \\
\cmidrule(lr){2-4}\cmidrule(lr){5-7}
& ECR & GSC$\uparrow$ & APSS$\downarrow$
& ECR & GSC$\uparrow$ & APSS$\downarrow$ \\
\midrule
\textbf{TopK} &
73.4 $\pm$ 0.3 & 55.9 $\pm$ 1.6 & 1.00 $\pm$ 0.00 &
73.4 $\pm$ 0.3 & 55.9 $\pm$ 1.6 & 1.00 $\pm$ 0.00 \\
\textbf{ICP} &
69.5 $\pm$ 0.8 & 49.3 $\pm$ 2.7 & 0.90 $\pm$ 0.02 &
64.5 $\pm$ 0.6 & 43.0 $\pm$ 2.1 & 0.78 $\pm$ 0.01 \\
\textbf{Learnt CP} &
69.5 $\pm$ 0.7 & 68.7 $\pm$ 1.3 & 0.97 $\pm$ 0.02 &
64.6 $\pm$ 0.7 & 63.0 $\pm$ 2.1 & 0.82 $\pm$ 0.02 \\
\textbf{CFC (ours)} &
68.4 $\pm$ 0.7 & 62.8 $\pm$ 2.0 & \textbf{0.88 $\pm$ 0.01} &
63.9 $\pm$ 0.7 & 59.2 $\pm$ 2.3 & \textbf{0.78 $\pm$ 0.01} \\
\textbf{CFC-P-F (ours)} &
70.0 $\pm$ 0.8 & \textbf{69.2 $\pm$ 1.3} & 0.99 $\pm$ 0.03 &
65.1 $\pm$ 0.8 & \textbf{63.7 $\pm$ 2.2} & 0.83 $\pm$ 0.02 \\
\bottomrule
\end{tabular}
\setlength{\tabcolsep}{6pt}
\end{table*}

\begin{table}[t]
\centering
\small
\caption{TriviaQA feature-map ablation at $\alpha=0.30$ (target coverage $70\%$). Each cell reports ECR/GSC/APSS. \textbf{TopK} and \textbf{ICP} are unchanged across feature maps and remain fixed at $73.4/55.9/1.00$ and $69.5/49.3/0.90$, respectively.}
\label{tab:trivia_split_ablation}
\setlength{\tabcolsep}{3pt}
\begin{tabular}{lccc}
\toprule
Setting & CFC & CFC-P-F \\
\midrule
Entropy-linear $\Phi$ & 66.9 / 45.1 / 1.02 & 70.6 / 53.2 / 1.32 \\
Max-loss-linear $\Phi$ & 67.6 / 56.5 / 0.99 & 70.7 / 57.4 / 1.26 \\
Chosen $\Phi$ & 68.4 / 62.8 / 0.88 & 70.0 / 69.2 / 0.99 \\
\bottomrule
\end{tabular}
\setlength{\tabcolsep}{6pt}
\end{table}

\begin{table}[t]
\centering
\small
\caption{GSM8K results at $\alpha=0.05$ using the first five sampled candidates per prompt and the quadratic basis $\Phi(X)=[1,T(X),T(X)^2]$ with $T(X)$ equal to mean verifier loss. ECR and GSC are reported in percent; for ECR, values closest to the target coverage $1-\alpha$ are preferred. Appendix~\ref{app:exp-details} reports the full sweep.}
\label{tab:ac2p_restyled}
\setlength{\tabcolsep}{3.0pt}
\begin{tabular}{l ccc}
\toprule
Methods & ECR & GSC$\uparrow$ & APSS$\downarrow$ \\
\midrule
\textbf{TopK} & 96.42 $\pm$ 0.39 & 86.52 $\pm$ 1.11 & \textbf{1.00 $\pm$ 0.00} \\
\textbf{ICP} & 95.09 $\pm$ 1.42 & 79.85 $\pm$ 6.53 & 4.73 $\pm$ 0.09 \\
\textbf{Learnt CP} & 94.91 $\pm$ 1.03 & 88.48 $\pm$ 3.05 & 4.01 $\pm$ 0.98 \\
\textbf{CFC (ours)} & 94.82 $\pm$ 0.97 & 88.48 $\pm$ 2.32 & 2.35 $\pm$ 0.43 \\
\textbf{CFC-P-F (ours)} & 95.24 $\pm$ 1.40 & \textbf{88.79 $\pm$ 3.01} & 4.59 $\pm$ 0.62 \\
\bottomrule
\end{tabular}
\setlength{\tabcolsep}{6pt}
\end{table}

\begin{table}[t]
\centering
\small
\caption{GSM8K sample-budget ablation at $\alpha=0.10$ under the same mean-loss quadratic rule. The larger budget gives only a small subgroup-coverage gain, but it substantially inflates APSS.}
\label{tab:gsm8k_n_ablation_main}
\setlength{\tabcolsep}{2pt}
\begin{tabular}{l ccc ccc}
\toprule
Method &
\multicolumn{3}{c}{$N=5$} &
\multicolumn{3}{c}{$N=20$} \\
\cmidrule(lr){2-4}\cmidrule(lr){5-7}
& ECR & GSC$\uparrow$ & APSS$\downarrow$
& ECR & GSC$\uparrow$ & APSS$\downarrow$ \\
\midrule
\textbf{CFC (ours)} & 90.18 & 86.36 & 1.49 & 90.30 & 87.73 & 3.79 \\
\textbf{CFC-P-F (ours)} & 91.91 & 88.48 & 2.34 & 92.06 & 88.79 & 7.97 \\
\bottomrule
\end{tabular}
\setlength{\tabcolsep}{6pt}
\end{table}

\begin{table}[t]
\centering
\small
\caption{VLM experiment on Flickr8k with \textsc{Qwen2-VL-7B-Instruct} at $\alpha=0.03$ using up to two cached candidates per image and the quadratic basis $\Phi(X)=[1,T(X),T(X)^2]$ with $T(X)$ equal to mean verifier loss. ECR and GSC are reported in percent; for ECR, values closest to the target coverage $1-\alpha=97\%$ are preferred.}
\label{tab:flickr8k_main}
\setlength{\tabcolsep}{3pt}
\begin{tabular}{lccc}
\toprule
Method & ECR & APSS$\downarrow$ & GSC$\uparrow$ \\
\midrule
\textbf{TopK} & 96.37 $\pm$ 0.17 & 1.00 $\pm$ 0.00 & 93.23 $\pm$ 0.47 \\
\textbf{ICP} & 95.58 $\pm$ 0.54 & 1.84 $\pm$ 0.01 & 85.21 $\pm$ 3.14 \\
\textbf{Learnt CP} & 96.16 $\pm$ 0.17 & 1.22 $\pm$ 0.07 & 94.48 $\pm$ 0.26 \\
\textbf{CFC (ours)} & 95.81 $\pm$ 0.38 & \textbf{0.99 $\pm$ 0.00} & 93.23 $\pm$ 0.47 \\
\textbf{CFC-P-F (ours)} & 97.27 $\pm$ 0.21 & 1.42 $\pm$ 0.07 & \textbf{95.21 $\pm$ 0.77} \\
\bottomrule
\end{tabular}
\setlength{\tabcolsep}{6pt}
\end{table}

\subsection{Real-World Data}
\label{sec:real}

\paragraph{Datasets.}
We evaluate on three settings. \textbf{GSM8K}~\cite{cobbe2021training} contains grade-school math word problems. \textbf{TriviaQA}~\cite{joshi2017triviaqa} is an open-domain question-answering benchmark. To test transfer beyond text-only generators, we also include a \textbf{Flickr8k} vision-language experiment with \textsc{Qwen2-VL-7B-Instruct} as the base model. We use \textsc{Llama-3-8B-Instruct} for GSM8K and TriviaQA.

\paragraph{Baselines and metrics.}
On real data we compare three baselines against CFC variants. \textbf{TopK} keeps the smallest top-$K$ prefix whose calibration coverage reaches the target level. \textbf{ICP} calibrates a single global verifier threshold. \textbf{Learnt CP} learns a feature-conditional threshold from calibration data but omits the exact conformal correction. We again report ECR, APSS, and GSC. For ECR, the preferred value is the one closest to the target coverage $1-\alpha$. APSS can fall below $1$ because the method may abstain and return the empty set on some prompts when no sampled candidate passes the calibrated threshold.

\paragraph{Implementation details.}
Unless otherwise stated, we use \textsc{Llama-3-8B-Instruct} as the base LLM with nucleus sampling (top-$p=0.8$, temperature $=0.7$) and draw up to $M=20$ samples per prompt. On reasoning tasks we use a separate verifier to compute $V(X,y)$, e.g. \textsc{Qwen2.5-Math-RM-72B} (and \textsc{Qwen2-VL-7B-Instruct} for VLM tasks). On GSM8K, the strongest setting uses the first $5$ sampled candidates per prompt, defines $T(X)$ as the mean verifier loss across those samples, and uses the quadratic basis $\Phi(X)=[1,T(X),T(X)^2]$. On Flickr8k, the strongest setting keeps up to $2$ cached candidates per image, defines $T(X)$ as the mean verifier loss across those candidates, and uses the quadratic basis $\Phi(X)=[1,T(X),T(X)^2]$. On TriviaQA, the strongest setting uses a calibration-defined feature map built from answer-distribution entropy and verifier loss; Appendix~\ref{app:exp-details} gives the exact construction. In the main paper, GSM8K, TriviaQA, and Flickr8k report \textbf{CFC}, and \textbf{CFC-P-F}. CFC truncates the accepted set after the best accepted candidate in sample order, while CFC-P-F applies the stability-based PAC adjustment to the full thresholded set. Appendix~\ref{app:exp-details} includes the full TriviaQA, GSM8K, and Flickr8k sweeps with all four CFC variants. For GSM8K and Flickr8k, GSC is computed over five equal-frequency bins of the scalar difficulty proxy; for TriviaQA, GSC uses the two-group feature map, whose exact definition is given in Appendix~\ref{app:exp-details}. All variants are purely post-hoc: they do not finetune the base generator or verifier.

\paragraph{Groupwise behavior.}
Figure~\ref{fig:real_group_miscoverage} complements Tables~\ref{tab:trivia_compact}--\ref{tab:flickr8k_main} by visualizing the same subgroup effect behind the real-data tables: a single global cutoff tends to under-cover the hardest inputs, while conditional thresholds flatten the miscoverage profile. This pattern is strongest on GSM8K, where the gain is concentrated on the hardest bins, and it remains visible on TriviaQA and Flickr8k as well.

\paragraph{Results for TriviaQA.}
Table~\ref{tab:trivia_compact} shows that the chosen TriviaQA feature map gives the strongest overall tradeoff we found, and the left panel of Figure~\ref{fig:real_group_miscoverage} shows why: the conditional rule mainly corrects the hard subset rather than trying to smooth every prompt equally. \textbf{Learnt CP} already improves subgroup reliability substantially over ICP, but it does so with larger prediction sets than \textbf{CFC}. At $\alpha=0.30$ (target coverage $70\%$), \textbf{CFC-P-F} is closest to target at $70.0\%$, while \textbf{CFC} attains the smallest prediction sets at $0.88$ and still raises GSC from $49.3\%$ (ICP) to $62.8\%$. The same division of labor is visible at $\alpha=0.20$, $\alpha=0.25$, and $\alpha=0.35$: the PAC variant is the most target-calibrated among our methods, while \textbf{CFC} is the most size-efficient.

\paragraph{Feature-map ablation.}
Table~\ref{tab:trivia_split_ablation} compares simple TriviaQA feature maps at $\alpha=0.30$. The chosen $\Phi$ yields the best subgroup reliability and the smallest APSS for \textbf{CFC}, and it also gives the strongest GSC for \textbf{CFC-P-F} while keeping target fit competitive. This is the main point of the ablation: the gains are not tied to a single scalar proxy, but the chosen feature map gives the best overall balance among target fit, subgroup reliability, and set size.

\paragraph{Results for GSM8K.}
Table~\ref{tab:ac2p_restyled}, Table~\ref{tab:gsm8k_n_ablation_main}, and the middle panel of Figure~\ref{fig:real_group_miscoverage} show that on GSM8K a small candidate budget together with a smooth loss-based feature map gives the best tradeoff we found. At $\alpha=0.05$ (target coverage $95\%$), all methods are reasonably close to target, but the conditional methods materially improve subgroup behavior over \textbf{ICP}: \textbf{CFC} reduces APSS from $4.73$ to $2.35$ while lifting GSC from $79.85\%$ to $88.48\%$, and \textbf{CFC-P-F} raises the subgroup floor further to $88.79\%$. Table~\ref{tab:gsm8k_n_ablation_main} explains why the main paper keeps the smaller $N=5$ budget: moving to $N=20$ barely changes ECR or GSC, but it more than doubles APSS for \textbf{CFC} and more than triples it for \textbf{CFC-P-F}. Appendix~\ref{app:exp-details} reports the complete GSM8K sweep.

\paragraph{Results for Flickr8k.}
Table~\ref{tab:flickr8k_main} and the right panel of Figure~\ref{fig:real_group_miscoverage} show that the same post-hoc conformal layer transfers to a vision-language model and a different verifier. At the target coverage level $1-\alpha=97\%$, \textbf{CFC-P-F} is closest to target at $97.27\%$ while achieving the strongest subgroup reliability at $95.21\%$ GSC. \textbf{CFC} is the most size-efficient variant at $0.99$ APSS, but on this easy benchmark it often collapses to a single caption and therefore gives up more target fit than the PAC rule. This still supports the transfer claim: the same conditional calibration layer remains effective when the base model is \textsc{Qwen2-VL-7B-Instruct}, and Appendix~\ref{app:exp-details} shows the full small-$\alpha$ sweep together with a compact setting ablation.

\paragraph{Additional Sweeps and Ablations.}
Appendix~\ref{app:exp-details} reports the full TriviaQA, GSM8K, and Flickr8k target-risk sweeps, other CFC variants, and compact ablations for the GSM8K budget choice and the Flickr8k setting choice.

\section{Conclusion}
Conditional Factuality Control replaces a single global factuality threshold with a feature-conditional one, and our results show that this post-hoc conformal layer improves subgroup reliability and often yields a better coverage--set-size tradeoff than marginal baselines across synthetic and real LLM/VLM settings.

\newpage
\section*{Acknowledgements}
This work was supported in part by the National Institutes of Health (NIH) under Grants R01HL173186 and R01HL177813, and by the National Science Foundation (NSF) under Grant No.~2306545.


{
    \small
    \bibliographystyle{ieeenat_fullname}
    \bibliography{main}
}

\clearpage
\onecolumn
\appendix
\section{Extended Background}\label{app:background}

\subsection{Conformal Factuality}

Let $X \in \mathcal X$ be a prompt and let
$\pi:\mathcal X \to \Delta(\mathcal Y)$ denote a fixed generator over completions.
At inference time, we repeatedly draw candidates $Y \sim \pi(\cdot \mid X)$ and seek a prediction set that contains \emph{at least one correct answer} with high probability:
\begin{equation}
\mathbb{P}\!\left(\exists y \in \widehat C_\alpha(X_{N+1}) : A(X_{N+1},y)=1\right) \ge 1-\alpha,
\end{equation}
where $\alpha$ is the target error level and $A(X,y)\in\{0,1\}$ indicates whether candidate $y$ is correct for prompt $X$.
To achieve this guarantee, each sampled candidate is evaluated with a verifier score
$V:\mathcal X\times\mathcal Y\to[0,1]$.
In this paper, smaller verifier scores are better, so a calibrated acceptance rule amounts to choosing a threshold and retaining all candidates with score below it.

\subsection{Inductive Conformal Prediction}

Split conformal prediction transforms the outputs of a black-box model into valid prediction sets using a held-out calibration set~\cite{vovk2005algorithmic,shafer2008tutorial}.
Given calibration pairs $\{(X_i,Y_i)\}_{i=1}^{N}$ and a nonconformity score function $s:\mathcal X\times\mathcal Y \to \mathbb R$, one computes calibration scores
$V_i=s(X_i,Y_i)$, sorts them, and forms the empirical quantile
\begin{equation}
\widehat Q_{1-\alpha} := V_{(\lceil (N+1)(1-\alpha)\rceil)}.
\end{equation}
The corresponding split-conformal prediction set is
\begin{equation}
\widehat C_{1-\alpha}(X_{N+1})
:=
\{y \in \mathcal Y : s(X_{N+1},y) \le \widehat Q_{1-\alpha}\},
\end{equation}
which satisfies marginal coverage
\begin{equation}
\mathbb P\!\left(Y_{N+1}\in \widehat C_{1-\alpha}(X_{N+1})\right)\ge 1-\alpha
\end{equation}
under exchangeability.
In the conformal-factuality setting, this corresponds to using one global acceptance threshold for all prompts.

\subsection{Conditional Conformal Prediction}

Marginal coverage holds only on average over the prompt distribution.
For LLMs, this can hide severe heterogeneity: easy prompts may be over-covered while hard prompts are under-covered.
The ideal pointwise conditional guarantee
\begin{equation}
\mathbb P\!\left(Y_{N+1}\in \widehat C(X_{N+1}) \mid X_{N+1}=x\right) \ge 1-\alpha
\end{equation}
is impossible to achieve exactly in finite samples without strong assumptions~\cite{vovk2012conditional,foygel2021limits}.

Following Gibbs et al.~\cite{gibbs2024conformal}, one can rewrite exact conditional coverage as an infinite family of weighted marginal constraints:
\begin{equation}
\mathbb P\!\left(Y_{N+1}\in \widehat C(X_{N+1}) \mid X_{N+1}\right)=1-\alpha
\iff
\mathbb E\!\left[f(X_{N+1})\big(\mathds{1}\{Y_{N+1}\in \widehat C(X_{N+1})\}-(1-\alpha)\big)\right]=0
\end{equation}
for all measurable $f$.
Their relaxation replaces the class of all measurable functions with a chosen function class $\mathcal F$:
\begin{equation}
\mathbb E\!\left[
f(X_{N+1})\big(\mathds{1}\{Y_{N+1}\in \widehat C(X_{N+1})\}-(1-\alpha)\big)
\right]=0,
\qquad \text{for all } f\in\mathcal F.
\end{equation}
Taking $\mathcal F=\{1\}$ recovers marginal conformal prediction, while
$\mathcal F=\{\Phi(\cdot)^\top\beta:\beta\in\mathbb R^d\}$ yields a finite-dimensional feature-conditional target.

For this linear class, Gibbs et al. define an augmented quantile-regression estimator using the pinball loss
\begin{equation}
\rho_{1-\alpha}(u) = u\big((1-\alpha)-\mathds{1}\{u<0\}\big).
\end{equation}
Given calibration scores $\{(X_i,S_i)\}_{i=1}^N$ and a fresh candidate score $S$, the augmented fit is
\begin{equation}
\widehat g_S
:=
\argmin_{g\in\mathcal F}
\frac{1}{N+1}\sum_{i=1}^{N}\rho_{1-\alpha}\!\big(S_i-g(X_i)\big)
\;+\;
\frac{1}{N+1}\rho_{1-\alpha}\!\big(S-g(X_{N+1})\big).
\end{equation}
The resulting prediction rule keeps labels whose score does not exceed the fitted value at the same score:
\begin{equation}
\widehat C(X_{N+1})
:=
\left\{
y : S(X_{N+1},y) \le \widehat g_{S(X_{N+1},y)}(X_{N+1})
\right\}.
\end{equation}

\begin{theorem}[Gibbs et al.~\cite{gibbs2024conformal}, Theorem 2]
Let $\mathcal F=\{\Phi(\cdot)^\top\beta:\beta\in\mathbb R^d\}$ be a linear class over the basis $\Phi:\mathcal X\to\mathbb R^d$.
Then for any non-negative $f\in\mathcal F$ with $\mathbb E[f(X)]>0$, the prediction rule above satisfies
\begin{equation}
\mathbb P_f\!\left(Y_{N+1}\in \widehat C(X_{N+1})\right)\ge 1-\alpha.
\end{equation}
\end{theorem}

Our method is a conformal-factuality instantiation of this framework, with the latent success score $S(X)$ replacing the standard label-wise nonconformity score and a fixed-point construction used to obtain the deployed threshold.

\section{Proof of Results from the Main Paper}\label{app:proofs}

\subsection{Proof of Theorem~\ref{thm:cfc_cov}}

\begin{proof}
Recall the success score
\[
S(X)
\;:=\;
\inf\{\lambda\in[0,1]:\ \ell_\lambda(X)=0\}
\;=\;
\inf\{V(X,y): y\in C(X),\ A(X,y)=1\},
\]
so that for any threshold $\lambda(X)\in[0,1]$,
\begin{equation}
\big\{\exists\,y\in\widehat C_\alpha(X): A(X,y)=1\big\}
\quad\Longleftrightarrow\quad
\big\{S(X)\le \lambda(X)\big\}.
\label{eq:succ-vs-S}
\end{equation}
Thus, if we show that
\begin{equation}
\label{eq:thm41-target}
\mathbb{P}_f\big(S(X_{N+1})\le \widehat\lambda_\alpha(X_{N+1})\big)\ \ge\ 1-\alpha
\end{equation}
for every non-negative $f\in\mathcal F$ with $\mathbb E[f(X)]>0$, the result follows immediately from~\eqref{eq:succ-vs-S}.

Throughout, for any such $f$ we define the $f$-reweighted probability of an event $E$ by
\[
\mathbb{P}_f(E)
:= \frac{\mathbb{E}\big[f(X)\,\mathds{1}\{(X,S)\in E\}\big]}{\mathbb{E}[f(X)]}.
\]

Let $\tau:=1-\alpha$ and recall the pinball loss
\[
\rho_\tau(u) \;=\; u\big(\tau-\mathds 1\{u<0\}\big).
\]
Let $\beta_S$ be the augmented quantile-regression minimizer from Eq.~\eqref{app:eq:aug-qr-obj}
for the realized calibration set and the fresh pair $(X,S)$, and define
$g(S):=\Phi(X)^\top \beta_S$.
Then for any nonnegative $f\in\mathcal F$ with $\mathbb E[f(X)]>0$,
\[
\mathbb{E}\big[f(X)\,\mathds 1\{S\le g(S)\}\big]\ \ge\ (1-\alpha)\,\mathbb{E}[f(X)].
\]

Fix $\gamma\in\mathbb R^d$ with $f(X)=\Phi(X)^\top\gamma\ge 0$ for all $X$ and consider
\[
\psi_N(\varepsilon)
=\frac{1}{N+1}\sum_{i=1}^N \rho_{1-\alpha}\big(S_i-\Phi(X_i)^\top(\beta_S+\varepsilon\gamma)\big)
+\frac{1}{N+1}\rho_{1-\alpha}\big(S-\Phi(X)^\top(\beta_S+\varepsilon\gamma)\big).
\]
By convexity and optimality of $\beta_S$, $\psi_N'(0^+)\ge 0$. This gives
\[
\psi_N'(0^+)
= -\frac{1}{N+1}\!\left[
\sum_{i=1}^{N}\!\big((1-\alpha)-\mathds 1\{S_i\le \Phi(X_i)^\top\beta_S\}\big) f(X_i)
+\big((1-\alpha)-\mathds 1\{S\le \Phi(X)^\top\beta_S\}\big) f(X)
\right]\!\ge 0.
\]
Rearranging,
\[
\frac{1}{N+1}\sum_{i=1}^{N}\!\mathds 1\{S_i\le \Phi(X_i)^\top\beta_S\} f(X_i)
+\frac{1}{N+1}\mathds 1\{S\le \Phi(X)^\top\beta_S\} f(X)
\ \ge\
(1-\alpha)\,\frac{1}{N+1}\sum_{i=1}^{N} f(X_i) \; + \; (1-\alpha)\,\frac{1}{N+1} f(X).
\]
Now take expectation over all $N{+}1$ exchangeable draws. By exchangeability, each of the $N{+}1$ summands on each side has the same distribution, so
\[
\mathbb E\big[f(X)\,\mathds 1\{S\le \Phi(X)^\top\beta_S\}\big]\ \ge\ (1-\alpha)\,\mathbb E[f(X)].
\]
Since $g(S)=\Phi(X)^\top\beta_S$, the claim follows.

Recall the CFC threshold
\[
\widehat\lambda_\alpha(X)\ :=\ \sup\{\,t\in[0,1]:\ t\le g(t)\,\},
\qquad g(t)=\Phi(X)^\top \beta_t.
\]
By definition of the supremum, for any realized $S\in[0,1]$ we have
\begin{equation}
\label{eq:incl}
\mathds 1\{S\le g(S)\}\ \le\ \mathds 1\{S\le \widehat\lambda_\alpha(X)\}.
\end{equation}

Multiplying \eqref{eq:incl} by $f(X)\ge 0$ and taking expectations,
\[
\mathbb E\big[f(X)\,\mathds 1\{S\le \widehat\lambda_\alpha(X)\}\big]
\ \ge\
\mathbb E\big[f(X)\,\mathds 1\{S\le g(S)\}\big]
\ \ge\ (1-\alpha)\,\mathbb E[f(X)].
\]
Dividing by $\mathbb E[f(X)]>0$ gives
\[
\mathbb P_f\big(S\le \widehat\lambda_\alpha(X)\big)\ \ge\ 1-\alpha.
\]

By definition of the success score and the prediction set,
\[
\{\,S(X)\le \widehat\lambda_\alpha(X)\,\}
\quad\Longleftrightarrow\quad
\{\,\exists\,y\in \widehat C_\alpha(X): A(X,y)=1\,\}.
\]
Therefore,
\[
\mathbb P_f\!\left(\exists\,y\in \widehat C_\alpha(X): A(X,y)=1\right)
\;=\;
\mathbb P_f\!\left(S(X)\le \widehat\lambda_\alpha(X)\right)
\;\ge\; 1-\alpha,
\]
as claimed.
\end{proof}

\subsection{Proof of Theorem~\ref{thm:pac-cfc}}

\begin{proof}
Write the calibration set as
\[
\mathcal D_{\mathrm{cal}}
= \{(X_i,S_i)\}_{i=1}^N,
\qquad
(X_{N+1},S_{N+1})\sim\text{i.i.d. as }(X_i,S_i),
\]
and recall the success indicator
\[
Z(x,s)\;:=\;\mathds{1}\bigl\{s\le \widehat{\lambda}_\alpha(x)\bigr\},
\]
so that $Z(X,S)=1$ iff there exists a correct candidate in $\widehat C_\alpha(X)$.

Define
\[
Q(\mathcal D_{\mathrm{cal}})
\;:=\;
\mathbb P\!\bigl(Z(X,S)=1 \mid \mathcal D_{\mathrm{cal}}\bigr)
\;=\;
\mathbb E\!\bigl[Z(X,S)\mid \mathcal D_{\mathrm{cal}}\bigr].
\]

By Theorem~\ref{thm:cfc_cov} applied with $f\equiv 1$ and the law of total expectation,
\[
\mathbb P\bigl(Z(X,S)=1\bigr)
= \mathbb E_{\mathcal D_{\mathrm{cal}}}\,Q(\mathcal D_{\mathrm{cal}})
\;\ge\; 1-\alpha.
\]
Thus
\begin{equation}
\label{eq:Q-mean-lower}
\mathbb E_{\mathcal D_{\mathrm{cal}}}\bigl[Q(\mathcal D_{\mathrm{cal}})\bigr]
\;\ge\; 1-\alpha.
\end{equation}

We now show that $Q(\mathcal D_{\mathrm{cal}})$ is Lipschitz in the calibration set, with sensitivity of order $1/N$ to replacing one calibration pair.

Let $\mathcal D=\{(X_i,S_i)\}_{i=1}^N$ and $\mathcal D'=\{(X_i',S_i')\}_{i=1}^N$ differ only in the $k$-th pair. For a fixed test prompt $x$ and a candidate success score $s\in[0,1]$, consider the ridge-regularized augmented quantile regression objective
\begin{equation}
\label{app:eq:aug-qr-obj}
\beta \;\mapsto\;
\frac{1}{N+1}\sum_{i=1}^N \rho_{1-\alpha}\!\left(S_i - \Phi(X_i)^\top \beta\right)
+\frac{1}{N+1}\rho_{1-\alpha}\!\left(s - \Phi(x)^\top \beta\right)
+\frac{\lambda}{2}\|\beta\|_2^2.
\end{equation}
Let $\beta_s(\mathcal D,x)$ and $\beta_s(\mathcal D',x)$ denote the unique minimizers of~\eqref{app:eq:aug-qr-obj} with $\mathcal D$ and $\mathcal D'$ respectively, and define
\[
g_{\mathcal D}(x,s)\ :=\ \Phi(x)^\top\beta_s(\mathcal D,x),
\qquad
g_{\mathcal D'}(x,s)\ :=\ \Phi(x)^\top\beta_s(\mathcal D',x).
\]

By Assumption~(1), $\|\Phi(X_i)\|_2\le R$ almost surely, and the pinball loss is 1-Lipschitz. Hence each sample loss is $R$-Lipschitz in $\beta$. By Assumption~(2), adding the ridge term makes the objective~\eqref{app:eq:aug-qr-obj} strongly convex. Standard stability results for regularized ERM imply that replacing one data point perturbs the minimizer by at most
\[
\bigl\|\beta_s(\mathcal D,x)-\beta_s(\mathcal D',x)\bigr\|_2
\;\le\; \frac{C_1 R}{\lambda N}
\]
for some universal constant $C_1>0$, uniformly over $s\in[0,1]$ and $x$.

Using Assumption~(1) again,
\[
\bigl|g_{\mathcal D}(x,s)-g_{\mathcal D'}(x,s)\bigr|
\le
\|\Phi(x)\|_2\,\bigl\|\beta_s(\mathcal D,x)-\beta_s(\mathcal D',x)\bigr\|_2
\le
\frac{C_1 R^2}{\lambda N}.
\]
Thus there is a constant $C_2>0$ such that
\begin{equation}
\label{eq:g-stability}
\sup_{x\in\mathcal X}\sup_{s\in[0,1]}
\bigl|g_{\mathcal D}(x,s)-g_{\mathcal D'}(x,s)\bigr|
\;\le\; \frac{C_2}{N}.
\end{equation}

The CFC threshold at $x$ is
\[
\widehat{\lambda}_\alpha^{(\mathcal D)}(x)
\;:=\;
\sup\{\,s\in[0,1]:\ s\le g_{\mathcal D}(x,s)\,\},
\]
and similarly for $\widehat{\lambda}_\alpha^{(\mathcal D')}(x)$. Under the fixed-point construction in Eq.~\eqref{eq:pre_set}, this threshold is a Lipschitz functional of $s\mapsto g_{\mathcal D}(x,s)$ in the uniform norm: there exists a constant $C_3>0$ such that
\begin{equation}
\label{eq:lambda-stability}
\sup_{x\in\mathcal X}
\bigl|\widehat{\lambda}_\alpha^{(\mathcal D)}(x)
-\widehat{\lambda}_\alpha^{(\mathcal D')}(x)\bigr|
\;\le\; \frac{C_3}{N}.
\end{equation}

For a fresh test pair $(X,S)$,
\[
Q(\mathcal D)
= \mathbb P\bigl(S\le \widehat{\lambda}_\alpha^{(\mathcal D)}(X)\mid \mathcal D\bigr)
= \mathbb E\!\left[F_{S\mid X}\bigl(\widehat{\lambda}_\alpha^{(\mathcal D)}(X)\bigr)
\;\middle|\; \mathcal D\right],
\]
where $F_{S\mid X=x}(\cdot)$ is the conditional CDF of $S$ given $X=x$.

By Assumption~(3), for each $x$ the map $t\mapsto F_{S\mid X=x}(t)$ is $L$-Lipschitz on $[0,1]$. Combining this with~\eqref{eq:lambda-stability} yields
\[
\bigl|Q(\mathcal D)-Q(\mathcal D')\bigr|
\le
L\,\sup_{x\in\mathcal X}
\bigl|\widehat{\lambda}_\alpha^{(\mathcal D)}(x)
-\widehat{\lambda}_\alpha^{(\mathcal D')}(x)\bigr|
\le
\frac{C_4}{N},
\]
where $C_4:=LC_3$.
Thus $Q(\mathcal D_{\mathrm{cal}})$ satisfies bounded differences with constants $c_i=C_4/N$ for $i=1,\dots,N$.

By McDiarmid’s inequality, for any $\varepsilon>0$,
\[
\mathbb P\!\left(
Q(\mathcal D_{\mathrm{cal}})
\le \mathbb E_{\mathcal D_{\mathrm{cal}}}\bigl[Q(\mathcal D_{\mathrm{cal}})\bigr]
- \varepsilon
\right)
\le
\exp\!\left(
-\frac{2N\varepsilon^2}{C_4^2}
\right).
\]
Given $\delta\in(0,1)$, set
\[
\varepsilon_N(\delta)
:= \frac{C_4}{\sqrt{2N}}\sqrt{\log\frac{1}{\delta}}
= O\!\left(\sqrt{\frac{\log(1/\delta)}{N}}\right).
\]
Then with probability at least $1-\delta$ over $\mathcal D_{\mathrm{cal}}$,
\begin{equation}
\label{eq:Q-high-prob}
Q(\mathcal D_{\mathrm{cal}})
\;\ge\;
\mathbb E_{\mathcal D_{\mathrm{cal}}}\bigl[Q(\mathcal D_{\mathrm{cal}})\bigr]
- \varepsilon_N(\delta).
\end{equation}

Combining~\eqref{eq:Q-high-prob} with~\eqref{eq:Q-mean-lower} yields
\[
Q(\mathcal D_{\mathrm{cal}})
\;\ge\; 1-\alpha - \varepsilon_N(\delta)
\]
with probability at least $1-\delta$, i.e.
\[
\mathbb{P}\!\left( Z(X,S)=1 \;\middle|\; \mathcal D_{\mathrm{cal}} \right)
\;\ge\; 1-\alpha - \varepsilon_N(\delta),
\]
Finally, we have $\alpha_{\mathrm{eff}}=\max\{0,\alpha-\varepsilon_N(\delta)\}=\alpha-\varepsilon_N(\delta)$ (the slack is a small term) from algorithm~\ref{alg:infer-pac}, then,
\[
\mathbb P\!\left(S\le \widehat\lambda_{\alpha_{\mathrm{eff}}}(X)\mid \mathcal D_{\mathrm{cal}}\right)
\ge
1-{\alpha_{\mathrm{eff}}}-\varepsilon_N(\delta)
=
1-\alpha,
\]
\end{proof}

\subsection{Proof of Proposition~\ref{prop:oracle-eff}}



\begin{proof}
Let
\[
u_0:=1-\alpha,
\qquad
u_{\bar\lambda}(t):=F_t(\bar\lambda).
\]
Since $\mathbb P(S\le \bar\lambda)=1-\alpha$, we have
\[
\mathbb E\big[u_{\bar\lambda}(T)\big]
=
\mathbb E\big[F_T(\bar\lambda)\big]
=
1-\alpha
=
u_0.
\]
By Assumption 2, the map $t\mapsto u_{\bar\lambda}(t)$ is nonincreasing.

Now,
\[
\mathbb E\big[G_X(\lambda^\star(X))\big]
=
\mathbb E\big[G_T(q_\alpha(T))\big]
=
\mathbb E\big[C_T(u_0)\big],
\]
because $q_\alpha(t)=F_t^{-1}(u_0)$, and also
\[
\mathbb E\big[G_X(\bar\lambda)\big]
=
\mathbb E\big[G_T(\bar\lambda)\big]
=
\mathbb E\big[C_T(u_{\bar\lambda}(T))\big].
\]

Set
\[
a(t):=u_{\bar\lambda}(t)-u_0.
\]
Then $a(t)$ is nonincreasing in $t$ and
\[
\mathbb E[a(T)]=0.
\]

By convexity of $C_t(\cdot)$, for each $t$,
\[
C_t(u_{\bar\lambda}(t))
\ge
C_t(u_0)+\partial_u C_t(u_0)\,\big(u_{\bar\lambda}(t)-u_0\big).
\]
Hence
\[
C_t(u_{\bar\lambda}(t))-C_t(u_0)
\ge
m(t)\,a(t),
\qquad
m(t):=\partial_u C_t(u_0).
\]
By Assumption 4, $m(t)$ is nonincreasing in $t$. Since both $m(t)$ and $a(t)$
are nonincreasing, Chebyshev's rearrangement inequality gives
\[
\mathbb E[m(T)a(T)]
\ge
\mathbb E[m(T)]\,\mathbb E[a(T)]
=
0.
\]
Taking expectations in the previous convexity bound yields
\[
\mathbb E\big[C_T(u_{\bar\lambda}(T))\big]
-
\mathbb E\big[C_T(u_0)\big]
\ge
\mathbb E[m(T)a(T)]
\ge 0.
\]
Therefore,
\[
\mathbb E\big[G_X(\bar\lambda)\big]
=
\mathbb E\big[C_T(u_{\bar\lambda}(T))\big]
\ge
\mathbb E\big[C_T(u_0)\big]
=
\mathbb E\big[G_X(\lambda^\star(X))\big].
\]

For strictness, if $\mathbb P(q_\alpha(X)\neq \bar\lambda_\alpha)>0$, then by
strict monotonicity of each $F_t$ we also have
\[
\mathbb P\big(F_T(\bar\lambda_\alpha)\neq 1-\alpha\big)>0.
\]
If $C_t(\cdot)$ is strictly convex for almost every $t$, then the supporting-line
inequality is strict on a set of positive probability, which implies
\[
\mathbb E\big[G_X(\lambda^\star(X))\big]
<
\mathbb E\big[G_X(\bar\lambda_\alpha)\big].
\]

\end{proof}

\subsection{Proof of Theorem~\ref{thm:cfc-eff}}

\begin{proof}
Let $T=\psi(X)$ and set
\[
u_0:=1-\alpha,
\qquad
\lambda^\star(X)=q_\alpha(T)=F_T^{-1}(u_0).
\]
By Proposition~\ref{prop:oracle-eff},
\begin{equation}
\label{eq:oracle-efficiency}
\mathbb E\big[G_X(\lambda^\star(X))\big]
\le
\mathbb E\big[G_X(\bar\lambda_\alpha)\big],
\end{equation}
with strict inequality under the additional strictness assumptions stated there.

It remains to show that
\[
\mathbb E\big[G_X(\widehat\lambda_{\alpha,N}(X))\big]
\longrightarrow
\mathbb E\big[G_X(\lambda^\star(X))\big].
\]

Fix $t$. Since $F_t$ is continuous and strictly increasing on $[0,1]$, we have
\[
F_t^{-1}(F_t(\lambda))=\lambda
\qquad\text{for all }\lambda\in[0,1].
\]
By definition,
\[
C_t(u)=G_t(F_t^{-1}(u)),
\]
so for every $\lambda\in[0,1]$,
\[
G_t(\lambda)=C_t(F_t(\lambda)).
\]
Now $u\mapsto C_t(u)$ is convex and differentiable on $(0,1)$ by
Proposition~\ref{prop:oracle-eff}, hence continuous on $(0,1)$.
Since
\[
F_t(\lambda^\star(t))=F_t\big(F_t^{-1}(u_0)\big)=u_0=1-\alpha\in(0,1),
\]
it follows that $\lambda\mapsto G_t(\lambda)$ is continuous at
$\lambda^\star(t)$.

Now fix $x\in\mathcal X$. By the assumed uniform consistency,
\[
\big|\widehat\lambda_{\alpha,N}(x)-\lambda^\star(x)\big|
\le
\sup_{z\in\mathcal X}
\big|\widehat\lambda_{\alpha,N}(z)-\lambda^\star(z)\big|
\xrightarrow{p}0,
\]
so
\[
\widehat\lambda_{\alpha,N}(x)\xrightarrow{p}\lambda^\star(x).
\]
Since $G_x(\cdot)$ is continuous at $\lambda^\star(x)$, the continuous mapping
theorem yields
\[
G_x\big(\widehat\lambda_{\alpha,N}(x)\big)
\xrightarrow{p}
G_x\big(\lambda^\star(x)\big).
\]

For $\varepsilon>0$, define
\[
p_{N,\varepsilon}(x)
:=
\mathbb P_{\mathcal D_{\mathrm{cal}}}\!\left(
\left|
G_x\big(\widehat\lambda_{\alpha,N}(x)\big)
-
G_x\big(\lambda^\star(x)\big)
\right|>\varepsilon
\right).
\]
Then $p_{N,\varepsilon}(x)\to 0$ for every fixed $x$, and
$0\le p_{N,\varepsilon}(x)\le 1$.
Since the fresh test point $X$ is independent of the calibration sample,
\[
\mathbb P\!\left(
\left|
G_X\big(\widehat\lambda_{\alpha,N}(X)\big)
-
G_X\big(\lambda^\star(X)\big)
\right|>\varepsilon
\right)
=
\mathbb E_X\big[p_{N,\varepsilon}(X)\big]
\longrightarrow 0
\]
by dominated convergence. Therefore,
\[
G_X\big(\widehat\lambda_{\alpha,N}(X)\big)
\xrightarrow{p}
G_X\big(\lambda^\star(X)\big).
\]

Moreover,
\[
0\le G_X\big(\widehat\lambda_{\alpha,N}(X)\big)\le 1,
\qquad
0\le G_X\big(\lambda^\star(X)\big)\le 1,
\]
so the sequence is uniformly integrable. Hence convergence in probability
upgrades to convergence in $L^1$, and therefore
\[
\lim_{N\to\infty}
\mathbb E\big[G_X(\widehat\lambda_{\alpha,N}(X))\big]
=
\mathbb E\big[G_X(\lambda^\star(X))\big].
\]
Combining this with~\eqref{eq:oracle-efficiency} proves
\[
\lim_{N\to\infty}\mathbb E\big[G_X(\widehat\lambda_{\alpha,N}(X))\big]
\le
\mathbb E\big[G_X(\bar\lambda_\alpha)\big],
\]
with strict inequality under the additional strictness assumptions from
Proposition~\ref{prop:oracle-eff}.

Finally, conditional on $X$ and a threshold $\lambda$, each of the $M$ sampled
candidates is accepted with probability $G_X(\lambda)$, so
\[
\mathbb E\big[\,|\widehat C_{\alpha,N}(X)|\,\big|\,X,\widehat\lambda_{\alpha,N}(X)\big]
=
M\,G_X\big(\widehat\lambda_{\alpha,N}(X)\big).
\]
Taking expectations gives
\[
\mathbb E\big[\,|\widehat C_{\alpha,N}(X)|\,\big]
=
M\,\mathbb E\big[G_X(\widehat\lambda_{\alpha,N}(X))\big].
\]
Passing to the limit yields
\[
\lim_{N\to\infty}\mathbb E\big[|\widehat C_{\alpha,N}(X)|\big]
=
M\,\mathbb E\big[G_X(\lambda^\star(X))\big]
\le
M\,\mathbb E\big[G_X(\bar\lambda_\alpha)\big],
\]
with strict inequality under the same additional assumptions.

\end{proof}

\section{Additional experiments}\label{app:exp-details}

\subsection{Synthetic Data}

\paragraph{Setup details.}
All synthetic appendix results use the same clean synthetic generator as the main text. Each prompt has a scalar difficulty variable $T\in[0,1]$, we draw $M$ candidates, and define the prompt-level latent success score as $S(X)=\min\{V_j:A_j=1\}$, with $S(X)=1$ if no sampled candidate is correct. Unless varied explicitly in the ablations, we use $N_{\mathrm{cal}}=N_{\mathrm{test}}=10{,}000$ and $M=50$. Throughout the synthetic appendix, ECR and GSC use the ground-truth correctness labels rather than the surrogate event $S(X)\le \widehat{\lambda}(X)$. The main synthetic comparisons use 10 equal-frequency bins for GSC, the threshold-adaptation plot below uses 5 bins for readability, and CFC-PAC uses the same stability-mode adjustment with $\delta=0.90$ as in the main synthetic experiments.

\paragraph{Threshold adaptation versus prompt difficulty.}
Figure~\ref{fig:lambda_vs_difficulty} visualizes the learned threshold as a function of prompt difficulty in the updated synthetic run at $\alpha=0.10$ using 5 difficulty bins. As expected, CFC assigns stricter thresholds to easy prompts and looser thresholds to hard prompts. This is exactly the adaptive behavior that a single global-threshold baseline cannot express.

\begin{figure}[t]
  \centering
  \includegraphics[width=0.50\linewidth]{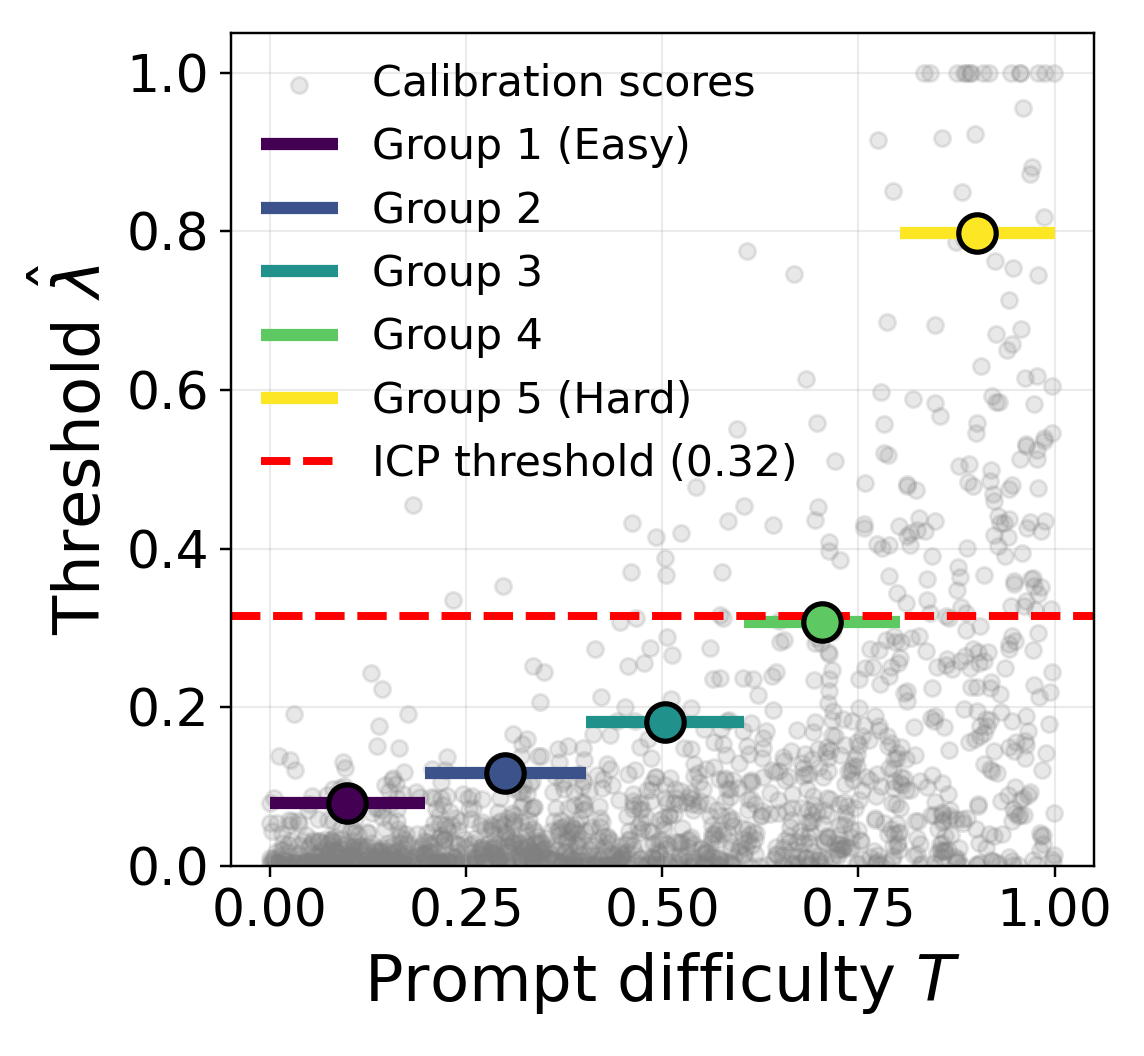}
  \caption{Learned threshold $\widehat{\lambda}_\alpha(X)$ versus prompt difficulty in the updated synthetic run ($\alpha=0.10$, 5 bins). Easy prompts receive stricter thresholds, while harder prompts receive looser thresholds, explaining the improved group-wise coverage of CFC relative to global-threshold baselines.}
  \label{fig:lambda_vs_difficulty}
\end{figure}

\paragraph{Full target-risk sweep.}
Table~\ref{tab:syn_cfc_new} reports the full synthetic sweep across target error rates. As above, ECR and GSC use ground-truth correctness labels, and CFC-PAC uses the stability-mode adjustment with $\delta=0.90$ throughout. We include \textsc{Learnt CP} in the full sweep to separate gains from learning a better threshold from gains due to exact conditional conformalization.

\begin{table*}[t]
\centering
\scriptsize
\renewcommand{\textbf}[1]{#1}
\caption{Results at different target error rates $\alpha$.}
\label{tab:syn_cfc_new}
\begin{tabular}{l ccc ccc}
\toprule
Methods &
\multicolumn{3}{c}{$\alpha = 0.10$} &
\multicolumn{3}{c}{$\alpha = 0.15$} \\
\cmidrule(lr){2-4}\cmidrule(lr){5-7}
& ECR & GSC$\uparrow$ & APSS$\downarrow$
& ECR & GSC$\uparrow$ & APSS$\downarrow$ \\
\midrule
\textbf{TopK} &
90.6 $\pm$ 0.1 & 58.2 $\pm$ 1.3 & 16.00 $\pm$ 0.00 &
85.0 $\pm$ 1.6 & 44.9 $\pm$ 2.3 & \textbf{10.80 $\pm$ 0.98} \\
\textbf{ICP} &
90.2 $\pm$ 0.2 & 57.4 $\pm$ 1.4 & 16.71 $\pm$ 0.23 &
85.3 $\pm$ 0.6 & 46.5 $\pm$ 1.4 & 12.11 $\pm$ 0.28 \\
\textbf{Learnt CP} &
90.2 $\pm$ 0.3 & 84.3 $\pm$ 0.5 & 15.72 $\pm$ 0.15 &
85.2 $\pm$ 0.6 & 77.4 $\pm$ 0.6 & 12.44 $\pm$ 0.17 \\
\textbf{CFC (ours)} &
90.3 $\pm$ 0.5 & 88.7 $\pm$ 0.7 & \textbf{15.53 $\pm$ 0.12} &
85.2 $\pm$ 0.6 & 82.7 $\pm$ 0.8 & 12.42 $\pm$ 0.09 \\
\textbf{CFC-PAC (ours)} &
90.8 $\pm$ 0.6 & \textbf{89.1 $\pm$ 0.5} & 15.87 $\pm$ 0.16 &
85.6 $\pm$ 0.6 & \textbf{83.4 $\pm$ 1.1} & 12.66 $\pm$ 0.07 \\
\midrule
Methods &
\multicolumn{3}{c}{$\alpha = 0.20$} &
\multicolumn{3}{c}{$\alpha = 0.25$} \\
\cmidrule(lr){2-4}\cmidrule(lr){5-7}
& ECR & GSC$\uparrow$ & APSS$\downarrow$
& ECR & GSC$\uparrow$ & APSS$\downarrow$ \\
\midrule
\textbf{TopK} &
79.7 $\pm$ 0.2 & 36.5 $\pm$ 0.9 & \textbf{8.00 $\pm$ 0.00} &
73.6 $\pm$ 0.2 & 29.4 $\pm$ 1.2 & \textbf{6.00 $\pm$ 0.00} \\
\textbf{ICP} &
80.3 $\pm$ 0.7 & 37.9 $\pm$ 1.6 & 9.32 $\pm$ 0.26 &
75.4 $\pm$ 0.8 & 31.6 $\pm$ 1.4 & 7.43 $\pm$ 0.21 \\
\textbf{Learnt CP} &
80.2 $\pm$ 0.6 & 71.2 $\pm$ 0.9 & 10.31 $\pm$ 0.09 &
75.4 $\pm$ 0.8 & 65.1 $\pm$ 1.6 & 8.72 $\pm$ 0.13 \\
\textbf{CFC (ours)} &
80.2 $\pm$ 0.6 & 77.4 $\pm$ 0.7 & 10.39 $\pm$ 0.07 &
75.3 $\pm$ 0.8 & 72.1 $\pm$ 1.2 & 8.80 $\pm$ 0.10 \\
\textbf{CFC-PAC (ours)} &
80.6 $\pm$ 0.6 & \textbf{78.1 $\pm$ 0.6} & 10.53 $\pm$ 0.07 &
75.7 $\pm$ 0.8 & \textbf{72.4 $\pm$ 1.2} & 8.89 $\pm$ 0.10 \\
\midrule
Methods &
\multicolumn{3}{c}{$\alpha = 0.30$} &
\multicolumn{3}{c}{$\alpha = 0.35$} \\
\cmidrule(lr){2-4}\cmidrule(lr){5-7}
& ECR & GSC$\uparrow$ & APSS$\downarrow$
& ECR & GSC$\uparrow$ & APSS$\downarrow$ \\
\midrule
\textbf{TopK} &
73.6 $\pm$ 0.2 & 29.4 $\pm$ 1.2 & 6.00 $\pm$ 0.00 &
64.2 $\pm$ 0.3 & 21.7 $\pm$ 1.4 & \textbf{4.00 $\pm$ 0.00} \\
\textbf{ICP} &
70.5 $\pm$ 0.8 & 27.1 $\pm$ 1.5 & \textbf{6.00 $\pm$ 0.19} &
65.6 $\pm$ 0.7 & 23.1 $\pm$ 1.0 & 4.91 $\pm$ 0.16 \\
\textbf{Learnt CP} &
70.4 $\pm$ 0.9 & 59.6 $\pm$ 1.7 & 7.37 $\pm$ 0.12 &
65.5 $\pm$ 1.0 & 54.7 $\pm$ 1.6 & 6.32 $\pm$ 0.14 \\
\textbf{CFC (ours)} &
70.4 $\pm$ 0.9 & 66.5 $\pm$ 1.4 & 7.53 $\pm$ 0.11 &
65.5 $\pm$ 0.8 & 61.3 $\pm$ 1.8 & 6.48 $\pm$ 0.10 \\
\textbf{CFC-PAC (ours)} &
70.7 $\pm$ 0.9 & \textbf{66.8 $\pm$ 1.3} & 7.60 $\pm$ 0.11 &
65.8 $\pm$ 0.9 & \textbf{61.5 $\pm$ 1.9} & 6.55 $\pm$ 0.11 \\
\midrule
Methods &
\multicolumn{3}{c}{$\alpha = 0.40$} &
\multicolumn{3}{c}{$\alpha = 0.45$} \\
\cmidrule(lr){2-4}\cmidrule(lr){5-7}
& ECR & GSC$\uparrow$ & APSS$\downarrow$
& ECR & GSC$\uparrow$ & APSS$\downarrow$ \\
\midrule
\textbf{TopK} &
64.2 $\pm$ 0.3 & 21.7 $\pm$ 1.4 & \textbf{4.00 $\pm$ 0.00} &
64.2 $\pm$ 0.3 & 21.7 $\pm$ 1.4 & 4.00 $\pm$ 0.00 \\
\textbf{ICP} &
60.8 $\pm$ 0.8 & 19.4 $\pm$ 0.7 & 4.06 $\pm$ 0.13 &
55.6 $\pm$ 1.0 & 16.5 $\pm$ 0.9 & \textbf{3.32 $\pm$ 0.13} \\
\textbf{Learnt CP} &
60.4 $\pm$ 0.9 & 49.6 $\pm$ 1.8 & 5.40 $\pm$ 0.13 &
55.7 $\pm$ 1.0 & 45.5 $\pm$ 1.6 & 4.66 $\pm$ 0.10 \\
\textbf{CFC (ours)} &
60.5 $\pm$ 0.7 & 56.4 $\pm$ 1.7 & 5.58 $\pm$ 0.08 &
55.6 $\pm$ 0.8 & 51.9 $\pm$ 1.5 & 4.83 $\pm$ 0.10 \\
\textbf{CFC-PAC (ours)} &
60.9 $\pm$ 0.7 & \textbf{56.6 $\pm$ 1.6} & 5.62 $\pm$ 0.09 &
56.0 $\pm$ 0.9 & \textbf{52.3 $\pm$ 1.5} & 4.88 $\pm$ 0.11 \\
\bottomrule
\end{tabular}
\end{table*}

\paragraph{Sensitivity to calibration size and sampling budget.}
Table~\ref{tab:ablation_bins10} reports a representative synthetic ablation of CFC and CFC-PAC as we vary the number of calibration points $N_{\text{cal}}$ and the sampling budget $M$, using the same 10-bin setting as the main synthetic comparison. For consistency with the main synthetic experiments, CFC-PAC uses the same stability-mode adjustment with $\delta=0.90$, and all reported coverages are true label-based coverages. Figure~\ref{fig:ablation_bins10} shows the corresponding group-level miscoverage profile.

\begin{table}[t]
\centering
\scriptsize
\renewcommand{\textbf}[1]{#1}
\caption{Ablation of CFC vs.\ CFC-PAC on synthetic data, with 10 bins.}
\label{tab:ablation_bins10}
\begin{tabular}{cc cccc}
\toprule
 & & \multicolumn{2}{c}{CFC} & \multicolumn{2}{c}{CFC-PAC} \\
\cmidrule(lr){3-4}\cmidrule(lr){5-6}
$N_{\text{cal}}$ & $M$ &
True cov. & Mean set &
True cov. & Mean set \\
\midrule
2000  & 50  & 0.799 $\pm$ 0.006 & 10.60 $\pm$ 0.35 & 0.808 $\pm$ 0.005 & 10.92 $\pm$ 0.30 \\
2000  & 100 & 0.809 $\pm$ 0.012 & 9.77 $\pm$ 0.31 & 0.819 $\pm$ 0.012 & 10.15 $\pm$ 0.29 \\
2000  & 150 & 0.796 $\pm$ 0.006 & 9.03 $\pm$ 0.43 & 0.807 $\pm$ 0.006 & 9.37 $\pm$ 0.43 \\
\midrule
5000  & 50  & 0.801 $\pm$ 0.004 & 10.45 $\pm$ 0.12 & 0.808 $\pm$ 0.004 & 10.69 $\pm$ 0.16 \\
5000  & 100 & 0.803 $\pm$ 0.008 & 9.52 $\pm$ 0.19 & 0.809 $\pm$ 0.007 & 9.74 $\pm$ 0.15 \\
5000  & 150 & 0.802 $\pm$ 0.006 & 9.13 $\pm$ 0.26 & 0.808 $\pm$ 0.007 & 9.35 $\pm$ 0.27 \\
\midrule
10000 & 50  & 0.802 $\pm$ 0.006 & 10.39 $\pm$ 0.07 & 0.806 $\pm$ 0.006 & 10.53 $\pm$ 0.07 \\
10000 & 100 & 0.803 $\pm$ 0.005 & 9.54 $\pm$ 0.06 & 0.808 $\pm$ 0.006 & 9.69 $\pm$ 0.09 \\
10000 & 150 & 0.799 $\pm$ 0.008 & 9.15 $\pm$ 0.15 & 0.803 $\pm$ 0.008 & 9.29 $\pm$ 0.15 \\
\bottomrule
\end{tabular}
\end{table}

\begin{figure}[t]
  \centering
  \includegraphics[width=0.75\linewidth]{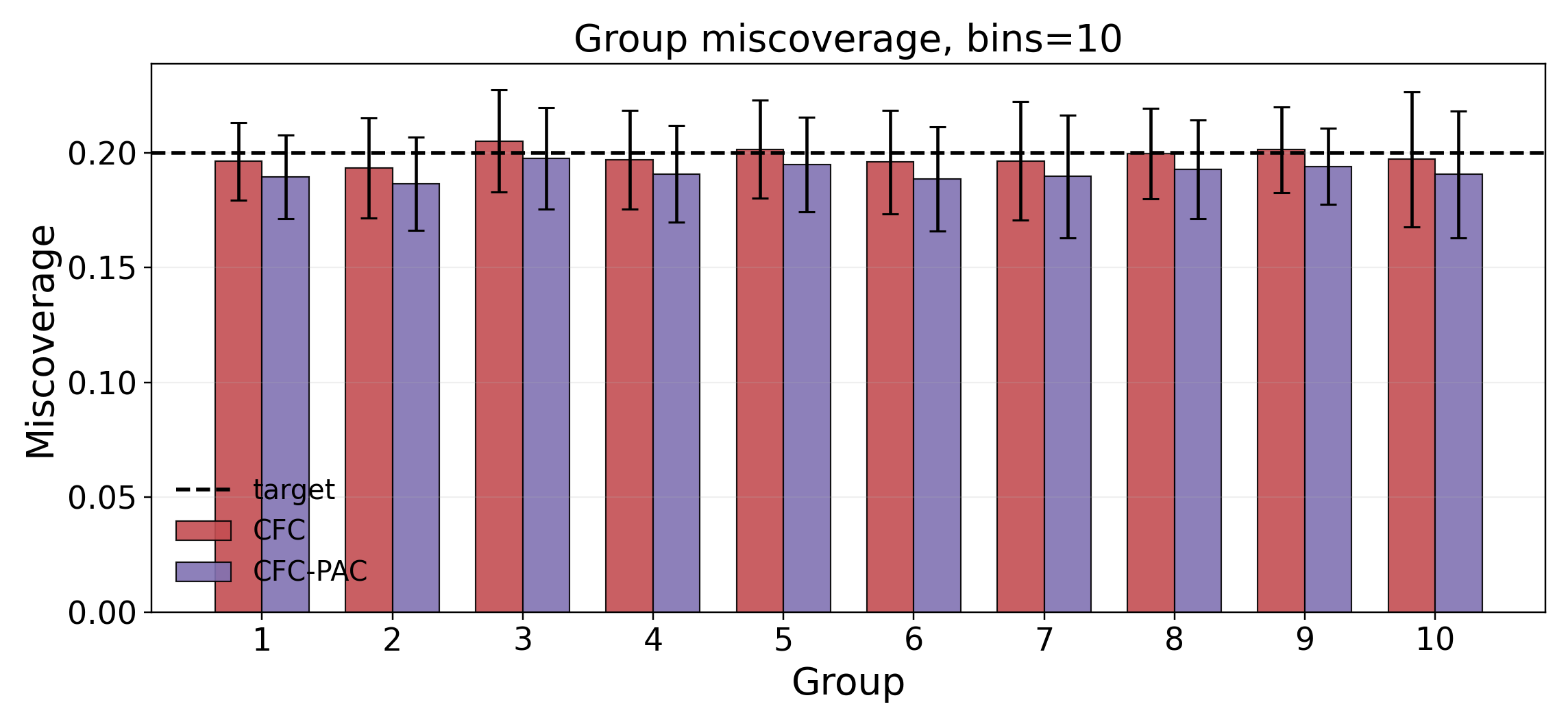}
  \caption{Group miscoverage for CFC and CFC-PAC, with 10 bins.}
  \label{fig:ablation_bins10}
\end{figure}

\subsection{Real-World Data}

\subsubsection{TriviaQA}

\paragraph{Full target-risk sweep.}
For the chosen TriviaQA feature map, we compute the rank-normalized answer-distribution entropy $T_{\mathrm{ent}}(X)$ and the rank-normalized maximum verifier loss $T_{\mathrm{loss}}(X)$ on the calibration split, then assign a prompt to the hard group when $\max\{T_{\mathrm{ent}}(X),T_{\mathrm{loss}}(X)\}\ge q_{0.925}$, where $q_{0.925}$ is the calibration $92.5$th percentile of that combined score. Table~\ref{tab:qrivalqa_full} reports the full TriviaQA sweep under this chosen feature map. The full table makes the main-paper tradeoff more explicit:  \textbf{CFC} is the most size-efficient of our methods, while \textbf{CFC-PAC-FULL} is the strongest at hitting the target coverage with a higher subgroup floor.

\begin{table*}[t]
\centering
\scriptsize
\renewcommand{\textbf}[1]{#1}
\caption{Full TriviaQA sweep across target error rates $\alpha$ under the calibration-defined split $\max\{T_{\mathrm{ent}}(X),T_{\mathrm{loss}}(X)\}\ge q_{0.925}$. APSS below $1$ indicates that the method abstains and returns the empty set on some prompts. ECR and GSC are reported in percent; for ECR, values closest to the target coverage $1-\alpha$ are preferred.}
\label{tab:qrivalqa_full}
\setlength{\tabcolsep}{3.5pt}
\begin{tabular}{l ccc ccc ccc}
\toprule
Methods &
\multicolumn{3}{c}{$\alpha = 0.20$} &
\multicolumn{3}{c}{$\alpha = 0.25$} &
\multicolumn{3}{c}{$\alpha = 0.30$} \\
\cmidrule(lr){2-4}\cmidrule(lr){5-7}\cmidrule(lr){8-10}
& ECR & GSC$\uparrow$ & APSS$\downarrow$
& ECR & GSC$\uparrow$ & APSS$\downarrow$
& ECR & GSC$\uparrow$ & APSS$\downarrow$ \\
\midrule
\textbf{TopK} &
81.6 $\pm$ 0.3 & 69.7 $\pm$ 1.5 & 1.75 $\pm$ 0.00 &
73.4 $\pm$ 0.3 & 55.9 $\pm$ 1.6 & 1.00 $\pm$ 0.00 &
73.4 $\pm$ 0.3 & 55.9 $\pm$ 1.6 & 1.00 $\pm$ 0.00 \\
\textbf{ICP} &
80.0 $\pm$ 0.5 & 65.1 $\pm$ 1.5 & 1.43 $\pm$ 0.02 &
74.9 $\pm$ 0.3 & 56.7 $\pm$ 1.9 & 1.08 $\pm$ 0.01 &
69.5 $\pm$ 0.8 & 49.3 $\pm$ 2.7 & 0.90 $\pm$ 0.02 \\
\textbf{Learnt CP} &
79.6 $\pm$ 0.5 & 76.3 $\pm$ 1.6 & 1.69 $\pm$ 0.04 &
74.7 $\pm$ 0.4 & 74.0 $\pm$ 1.1 & 1.22 $\pm$ 0.03 &
69.5 $\pm$ 0.7 & 68.7 $\pm$ 1.3 & 0.97 $\pm$ 0.02 \\
\textbf{CFC (ours)} &
76.2 $\pm$ 0.5 & 65.4 $\pm$ 1.7 & 1.21 $\pm$ 0.01 &
72.7 $\pm$ 0.4 & 65.2 $\pm$ 1.8 & 1.03 $\pm$ 0.03 &
68.4 $\pm$ 0.7 & 62.8 $\pm$ 2.0 & 0.88 $\pm$ 0.01 \\
\textbf{CFC-PAC (ours)} &
76.4 $\pm$ 0.5 & 65.4 $\pm$ 1.7 & 1.22 $\pm$ 0.01 &
73.1 $\pm$ 0.5 & 65.3 $\pm$ 1.8 & 1.05 $\pm$ 0.04 &
68.8 $\pm$ 0.7 & 62.9 $\pm$ 2.1 & 0.89 $\pm$ 0.02 \\
\textbf{CFC-FULL (ours)} &
79.6 $\pm$ 0.5 & 76.3 $\pm$ 1.6 & 1.69 $\pm$ 0.04 &
74.8 $\pm$ 0.3 & 74.1 $\pm$ 1.1 & 1.26 $\pm$ 0.09 &
69.6 $\pm$ 0.7 & 68.9 $\pm$ 1.1 & 0.97 $\pm$ 0.02 \\
\textbf{CFC-PAC-FULL (ours)} &
80.1 $\pm$ 0.5 & 76.3 $\pm$ 1.6 & 1.72 $\pm$ 0.04 &
75.3 $\pm$ 0.4 & 74.6 $\pm$ 1.0 & 1.32 $\pm$ 0.10 &
70.0 $\pm$ 0.8 & 69.2 $\pm$ 1.3 & 0.99 $\pm$ 0.03 \\
\midrule
Methods &
\multicolumn{3}{c}{$\alpha = 0.35$} &
\multicolumn{3}{c}{$\alpha = 0.40$} &
\multicolumn{3}{c}{$\alpha = 0.45$} \\
\cmidrule(lr){2-4}\cmidrule(lr){5-7}\cmidrule(lr){8-10}
& ECR & GSC$\uparrow$ & APSS$\downarrow$
& ECR & GSC$\uparrow$ & APSS$\downarrow$
& ECR & GSC$\uparrow$ & APSS$\downarrow$ \\
\midrule
\textbf{TopK} &
73.4 $\pm$ 0.3 & 55.9 $\pm$ 1.6 & 1.00 $\pm$ 0.00 &
73.4 $\pm$ 0.3 & 55.9 $\pm$ 1.6 & 1.00 $\pm$ 0.00 &
73.4 $\pm$ 0.3 & 55.9 $\pm$ 1.6 & 1.00 $\pm$ 0.00 \\
\textbf{ICP} &
64.5 $\pm$ 0.6 & 43.0 $\pm$ 2.1 & 0.78 $\pm$ 0.01 &
59.9 $\pm$ 0.9 & 36.1 $\pm$ 3.1 & 0.70 $\pm$ 0.01 &
54.8 $\pm$ 0.9 & 29.6 $\pm$ 2.3 & 0.62 $\pm$ 0.01 \\
\textbf{Learnt CP} &
64.6 $\pm$ 0.7 & 63.0 $\pm$ 2.1 & 0.82 $\pm$ 0.02 &
59.9 $\pm$ 0.8 & 58.2 $\pm$ 2.4 & 0.72 $\pm$ 0.01 &
55.2 $\pm$ 1.0 & 53.8 $\pm$ 1.7 & 0.65 $\pm$ 0.01 \\
\textbf{CFC (ours)} &
63.9 $\pm$ 0.7 & 59.2 $\pm$ 2.3 & 0.78 $\pm$ 0.01 &
59.5 $\pm$ 0.9 & 55.8 $\pm$ 2.7 & 0.70 $\pm$ 0.01 &
54.9 $\pm$ 1.0 & 52.3 $\pm$ 2.4 & 0.63 $\pm$ 0.01 \\
\textbf{CFC-PAC (ours)} &
64.3 $\pm$ 0.7 & 59.5 $\pm$ 2.4 & 0.78 $\pm$ 0.01 &
59.9 $\pm$ 0.8 & 56.4 $\pm$ 2.9 & 0.70 $\pm$ 0.01 &
55.4 $\pm$ 1.0 & 52.8 $\pm$ 2.6 & 0.64 $\pm$ 0.01 \\
\textbf{CFC-FULL (ours)} &
64.7 $\pm$ 0.7 & 63.2 $\pm$ 2.0 & 0.82 $\pm$ 0.02 &
59.9 $\pm$ 0.9 & 58.4 $\pm$ 2.2 & 0.72 $\pm$ 0.01 &
55.2 $\pm$ 1.0 & 53.8 $\pm$ 1.7 & 0.65 $\pm$ 0.01 \\
\textbf{CFC-PAC-FULL (ours)} &
65.1 $\pm$ 0.8 & 63.7 $\pm$ 2.2 & 0.83 $\pm$ 0.02 &
60.4 $\pm$ 0.8 & 59.0 $\pm$ 2.2 & 0.73 $\pm$ 0.01 &
55.7 $\pm$ 1.0 & 54.3 $\pm$ 1.9 & 0.65 $\pm$ 0.01 \\
\bottomrule
\end{tabular}
\setlength{\tabcolsep}{6pt}
\end{table*}

\subsubsection{GSM8K}

\paragraph{Full target-risk sweep.}
Table~\ref{tab:gsm8k_full} reports the full GSM8K target-risk sweep for the chosen setting from the main paper: we keep the first $5$ sampled candidates per prompt, define $T(X)$ as the mean verifier loss across those candidates, and use the quadratic basis $\Phi(X)=[1,T(X),T(X)^2]$. The same qualitative pattern holds across the sweep: the conditional methods remain much more efficient than ICP while sharply improving worst-group coverage.

\begin{table*}[t]
\centering
\scriptsize
\renewcommand{\textbf}[1]{#1}
\caption{Full GSM8K sweep across target error rates $\alpha$ using the first five sampled candidates per prompt, $T(X)$ equal to mean verifier loss, and the quadratic basis $\Phi(X)=[1,T(X),T(X)^2]$. ECR and GSC are reported in percent; for ECR, values closest to the target coverage $1-\alpha$ are preferred.}
\label{tab:gsm8k_full}
\setlength{\tabcolsep}{3.0pt}
\begin{tabular}{l ccc ccc ccc}
\toprule
Methods &
\multicolumn{3}{c}{$\alpha = 0.05$} &
\multicolumn{3}{c}{$\alpha = 0.10$} &
\multicolumn{3}{c}{$\alpha = 0.15$} \\
\cmidrule(lr){2-4}\cmidrule(lr){5-7}\cmidrule(lr){8-10}
& ECR & GSC$\uparrow$ & APSS$\downarrow$
& ECR & GSC$\uparrow$ & APSS$\downarrow$
& ECR & GSC$\uparrow$ & APSS$\downarrow$ \\
\midrule
\textbf{TopK} & 96.42 $\pm$ 0.39 & 86.52 $\pm$ 1.11 & 1.00 $\pm$ 0.00 & 96.42 $\pm$ 0.39 & 86.52 $\pm$ 1.11 & 1.00 $\pm$ 0.00 & 96.42 $\pm$ 0.39 & 86.52 $\pm$ 1.11 & 1.00 $\pm$ 0.00 \\
\textbf{ICP} & 95.09 $\pm$ 1.42 & 79.85 $\pm$ 6.53 & 4.73 $\pm$ 0.09 & 90.39 $\pm$ 1.44 & 56.36 $\pm$ 6.20 & 4.36 $\pm$ 0.08 & 83.91 $\pm$ 2.15 & 27.73 $\pm$ 8.14 & 3.97 $\pm$ 0.11 \\
\textbf{Learnt CP} & 94.91 $\pm$ 1.03 & 88.48 $\pm$ 3.05 & 4.01 $\pm$ 0.98 & 90.09 $\pm$ 1.38 & 86.06 $\pm$ 0.77 & 2.22 $\pm$ 0.07 & 84.70 $\pm$ 1.11 & 79.09 $\pm$ 1.56 & 1.92 $\pm$ 0.06 \\
\textbf{CFC (ours)} & 94.82 $\pm$ 0.97 & 88.48 $\pm$ 2.32 & 2.35 $\pm$ 0.43 & 90.18 $\pm$ 1.41 & 86.36 $\pm$ 1.07 & 1.49 $\pm$ 0.04 & 84.97 $\pm$ 1.12 & 79.55 $\pm$ 1.52 & 1.34 $\pm$ 0.04 \\
\textbf{CFC-FULL (ours)} & 95.03 $\pm$ 0.97 & 88.94 $\pm$ 2.74 & 4.08 $\pm$ 0.93 & 90.30 $\pm$ 1.40 & 86.36 $\pm$ 1.07 & 2.23 $\pm$ 0.08 & 85.09 $\pm$ 1.03 & 79.55 $\pm$ 1.52 & 1.96 $\pm$ 0.06 \\
\textbf{CFC-PAC (ours)} & 95.03 $\pm$ 1.33 & 88.18 $\pm$ 2.47 & 2.59 $\pm$ 0.29 & 91.79 $\pm$ 1.66 & 88.18 $\pm$ 1.41 & 1.55 $\pm$ 0.06 & 86.64 $\pm$ 0.95 & 81.67 $\pm$ 1.11 & 1.38 $\pm$ 0.03 \\
\textbf{CFC-PAC-FULL (ours)} & 95.24 $\pm$ 1.40 & 88.79 $\pm$ 3.01 & 4.59 $\pm$ 0.62 & 91.91 $\pm$ 1.68 & 88.48 $\pm$ 1.62 & 2.34 $\pm$ 0.11 & 86.76 $\pm$ 0.88 & 81.67 $\pm$ 1.11 & 2.05 $\pm$ 0.05 \\
\midrule
Methods &
\multicolumn{3}{c}{$\alpha = 0.20$} &
\multicolumn{3}{c}{$\alpha = 0.25$} &
\multicolumn{3}{c}{$\alpha = 0.30$} \\
\cmidrule(lr){2-4}\cmidrule(lr){5-7}\cmidrule(lr){8-10}
& ECR & GSC$\uparrow$ & APSS$\downarrow$
& ECR & GSC$\uparrow$ & APSS$\downarrow$
& ECR & GSC$\uparrow$ & APSS$\downarrow$ \\
\midrule
\textbf{TopK} & 96.42 $\pm$ 0.39 & 86.52 $\pm$ 1.11 & 1.00 $\pm$ 0.00 & 96.42 $\pm$ 0.39 & 86.52 $\pm$ 1.11 & 1.00 $\pm$ 0.00 & 96.42 $\pm$ 0.39 & 86.52 $\pm$ 1.11 & 1.00 $\pm$ 0.00 \\
\textbf{ICP} & 79.55 $\pm$ 2.91 & 18.48 $\pm$ 4.92 & 3.69 $\pm$ 0.16 & 74.64 $\pm$ 2.53 & 13.48 $\pm$ 1.30 & 3.38 $\pm$ 0.15 & 70.18 $\pm$ 2.51 & 10.15 $\pm$ 0.91 & 3.11 $\pm$ 0.13 \\
\textbf{Learnt CP} & 79.42 $\pm$ 1.87 & 71.36 $\pm$ 3.63 & 1.74 $\pm$ 0.10 & 74.82 $\pm$ 2.51 & 66.97 $\pm$ 2.90 & 1.57 $\pm$ 0.07 & 69.36 $\pm$ 2.48 & 60.91 $\pm$ 2.23 & 1.42 $\pm$ 0.08 \\
\textbf{CFC (ours)} & 79.94 $\pm$ 1.80 & 72.12 $\pm$ 3.85 & 1.22 $\pm$ 0.04 & 75.03 $\pm$ 2.52 & 67.12 $\pm$ 2.86 & 1.12 $\pm$ 0.05 & 69.55 $\pm$ 2.53 & 61.21 $\pm$ 2.11 & 1.01 $\pm$ 0.05 \\
\textbf{CFC-FULL (ours)} & 80.06 $\pm$ 1.78 & 72.12 $\pm$ 3.85 & 1.76 $\pm$ 0.08 & 75.15 $\pm$ 2.47 & 67.12 $\pm$ 2.86 & 1.60 $\pm$ 0.07 & 69.64 $\pm$ 2.48 & 61.21 $\pm$ 2.11 & 1.44 $\pm$ 0.09 \\
\textbf{CFC-PAC (ours)} & 81.36 $\pm$ 1.68 & 73.18 $\pm$ 4.27 & 1.25 $\pm$ 0.04 & 75.94 $\pm$ 2.47 & 68.03 $\pm$ 2.81 & 1.13 $\pm$ 0.05 & 70.82 $\pm$ 2.58 & 62.58 $\pm$ 2.01 & 1.04 $\pm$ 0.05 \\
\textbf{CFC-PAC-FULL (ours)} & 81.48 $\pm$ 1.65 & 73.18 $\pm$ 4.27 & 1.82 $\pm$ 0.08 & 76.06 $\pm$ 2.42 & 68.03 $\pm$ 2.81 & 1.63 $\pm$ 0.07 & 70.94 $\pm$ 2.52 & 62.58 $\pm$ 2.01 & 1.48 $\pm$ 0.08 \\
\bottomrule
\end{tabular}
\setlength{\tabcolsep}{6pt}
\end{table*}

\paragraph{Candidate-budget ablation.}
Table~\ref{tab:gsm8k_n_ablation} compares the chosen $N=5$ budget against $N=20$ at the representative target $\alpha=0.10$, keeping the same mean-loss proxy and quadratic basis. The larger candidate budget barely changes target calibration, but it inflates APSS substantially for every threshold-based method. This is why the main paper uses the smaller budget on GSM8K: the extra samples add little new diversity but materially hurt efficiency.

\begin{table*}[t]
\centering
\scriptsize
\renewcommand{\textbf}[1]{#1}
\caption{GSM8K sample-budget ablation at $\alpha=0.10$ under the mean-loss quadratic rule. Each budget uses the same basis $\Phi(X)=[1,T(X),T(X)^2]$, differing only in the number of retained sampled candidates per prompt.}
\label{tab:gsm8k_n_ablation}
\begin{tabular}{l ccc ccc}
\toprule
Methods &
\multicolumn{3}{c}{$N=5$} &
\multicolumn{3}{c}{$N=20$} \\
\cmidrule(lr){2-4}\cmidrule(lr){5-7}
& ECR & GSC$\uparrow$ & APSS$\downarrow$
& ECR & GSC$\uparrow$ & APSS$\downarrow$ \\
\midrule
\textbf{TopK} & 96.42 $\pm$ 0.39 & 86.52 $\pm$ 1.11 & 1.00 $\pm$ 0.00 & 96.70 $\pm$ 0.36 & 88.48 $\pm$ 1.11 & 1.00 $\pm$ 0.00 \\
\textbf{ICP} & 90.39 $\pm$ 1.44 & 56.36 $\pm$ 6.20 & 4.36 $\pm$ 0.08 & 90.15 $\pm$ 1.37 & 55.76 $\pm$ 5.52 & 16.75 $\pm$ 0.35 \\
\textbf{Learnt CP} & 90.09 $\pm$ 1.38 & 86.06 $\pm$ 0.77 & 2.22 $\pm$ 0.07 & 90.15 $\pm$ 0.81 & 87.42 $\pm$ 1.41 & 7.24 $\pm$ 0.33 \\
\textbf{CFC (ours)} & 90.18 $\pm$ 1.41 & 86.36 $\pm$ 1.07 & \textbf{1.49 $\pm$ 0.04} & 90.30 $\pm$ 0.63 & 87.73 $\pm$ 1.47 & \textbf{3.79 $\pm$ 0.12} \\
\textbf{CFC-FULL (ours)} & 90.30 $\pm$ 1.40 & 86.36 $\pm$ 1.07 & 2.23 $\pm$ 0.08 & 90.45 $\pm$ 0.70 & 87.73 $\pm$ 1.47 & 7.50 $\pm$ 0.18 \\
\textbf{CFC-PAC (ours)} & 91.79 $\pm$ 1.66 & 88.18 $\pm$ 1.41 & 1.55 $\pm$ 0.06 & 91.91 $\pm$ 0.46 & 88.64 $\pm$ 2.30 & 3.99 $\pm$ 0.07 \\
\textbf{CFC-PAC-FULL (ours)} & 91.91 $\pm$ 1.68 & 88.48 $\pm$ 1.62 & 2.34 $\pm$ 0.11 & 92.06 $\pm$ 0.57 & 88.79 $\pm$ 2.46 & 7.97 $\pm$ 0.03 \\
\bottomrule
\end{tabular}
\end{table*}

\subsubsection{Flickr8k}

\paragraph{Full target-risk sweep.}
Table~\ref{tab:flickr8k_vlm} reports the full Flickr8k sweep for the chosen clean setting from the main paper: we keep up to two cached candidates per image, define $T(X)$ as the mean verifier loss across those candidates, and use the quadratic basis $\Phi(X)=[1,T(X),T(X)^2]$. This benchmark is visibly easier than GSM8K or TriviaQA, so the most informative comparison is closeness to the target coverage together with subgroup reliability. At $\alpha=0.03$ (target coverage $97\%$), \textbf{CFC-PAC-FULL} is the closest full-set variant to target while improving GSC over every baseline, whereas  \textbf{CFC} collapses to almost one caption per image and is best viewed as the smallest-size extreme.

\begin{table*}[t]
\centering
\scriptsize
\renewcommand{\textbf}[1]{#1}
\caption{Full Flickr8k sweep with \textsc{Qwen2-VL-7B-Instruct} using up to two cached candidates per image, $T(X)$ equal to mean verifier loss, and the quadratic basis $\Phi(X)=[1,T(X),T(X)^2]$. ECR and GSC are reported in percent; for ECR, values closest to the target coverage $1-\alpha$ are preferred.}
\label{tab:flickr8k_vlm}
\setlength{\tabcolsep}{3.0pt}
\begin{tabular}{l ccc ccc ccc}
\toprule
Methods &
\multicolumn{3}{c}{$\alpha = 0.01$} &
\multicolumn{3}{c}{$\alpha = 0.02$} &
\multicolumn{3}{c}{$\alpha = 0.03$} \\
\cmidrule(lr){2-4}\cmidrule(lr){5-7}\cmidrule(lr){8-10}
& ECR & GSC$\uparrow$ & APSS$\downarrow$
& ECR & GSC$\uparrow$ & APSS$\downarrow$
& ECR & GSC$\uparrow$ & APSS$\downarrow$ \\
\midrule
\textbf{TopK} & 97.75 $\pm$ 0.19 & 95.62 $\pm$ 0.71 & 2.00 $\pm$ 0.00 & 97.75 $\pm$ 0.19 & 95.62 $\pm$ 0.71 & 2.00 $\pm$ 0.00 & 96.37 $\pm$ 0.17 & 93.23 $\pm$ 0.47 & 1.00 $\pm$ 0.00 \\
\textbf{ICP} & 97.29 $\pm$ 0.29 & 93.33 $\pm$ 1.45 & 1.93 $\pm$ 0.01 & 96.25 $\pm$ 0.66 & 88.23 $\pm$ 3.66 & 1.87 $\pm$ 0.02 & 95.58 $\pm$ 0.54 & 85.21 $\pm$ 3.14 & 1.84 $\pm$ 0.01 \\
\textbf{Learnt CP} & 97.39 $\pm$ 0.33 & 95.52 $\pm$ 0.85 & 1.74 $\pm$ 0.06 & 97.06 $\pm$ 0.31 & 95.10 $\pm$ 0.63 & 1.34 $\pm$ 0.07 & 96.16 $\pm$ 0.17 & 94.48 $\pm$ 0.26 & 1.22 $\pm$ 0.07 \\
\textbf{CFC (ours)} & 96.37 $\pm$ 0.17 & 93.23 $\pm$ 0.47 & \textbf{1.00 $\pm$ 0.00} & 96.37 $\pm$ 0.17 & 93.23 $\pm$ 0.47 & \textbf{1.00 $\pm$ 0.00} & 95.81 $\pm$ 0.38 & 93.23 $\pm$ 0.47 & \textbf{0.99 $\pm$ 0.00} \\
\textbf{CFC-FULL (ours)} & 97.66 $\pm$ 0.16 & 95.62 $\pm$ 0.71 & 1.86 $\pm$ 0.04 & 97.27 $\pm$ 0.21 & 95.21 $\pm$ 0.77 & 1.40 $\pm$ 0.06 & 96.27 $\pm$ 0.24 & 94.58 $\pm$ 0.26 & 1.25 $\pm$ 0.08 \\
\textbf{CFC-PAC (ours)} & 96.37 $\pm$ 0.17 & 93.23 $\pm$ 0.47 & \textbf{1.00 $\pm$ 0.00} & 96.37 $\pm$ 0.17 & 93.23 $\pm$ 0.47 & \textbf{1.00 $\pm$ 0.00} & 96.37 $\pm$ 0.17 & 93.23 $\pm$ 0.47 & 1.00 $\pm$ 0.00 \\
\textbf{CFC-PAC-FULL (ours)} & 97.75 $\pm$ 0.19 & 95.62 $\pm$ 0.71 & 2.00 $\pm$ 0.00 & 97.64 $\pm$ 0.13 & 95.62 $\pm$ 0.71 & 1.86 $\pm$ 0.06 & 97.27 $\pm$ 0.21 & 95.21 $\pm$ 0.77 & 1.42 $\pm$ 0.07 \\
\bottomrule
\end{tabular}
\setlength{\tabcolsep}{6pt}
\end{table*}

\paragraph{Setting ablation.}
Table~\ref{tab:flickr8k_ablation} compares the chosen Flickr8k setting against two nearby alternatives from the clean search at the representative target $\alpha=0.03$. The chosen $N=2$ mean-loss quadratic rule is the best-balanced option we found: it keeps the  \textbf{CFC} variant essentially at single-caption size, while \textbf{CFC-PAC-FULL} remains close to target without the larger APSS jump of the $N=3$ alternative.

\begin{table}[t]
\centering
\scriptsize
\renewcommand{\textbf}[1]{#1}
\caption{Flickr8k setting ablation at $\alpha=0.03$ (target coverage $97\%$). Each cell reports ECR/APSS/GSC for the corresponding method.}
\label{tab:flickr8k_ablation}
\begin{tabular}{lcc}
\toprule
Setting & CFC & CFC-PAC-FULL \\
\midrule
$N=2$, max-loss, poly2 & 96.02 / 1.00 / 93.02 & 97.62 / 1.89 / \textbf{95.93} \\
Chosen: $N=2$, mean-loss, poly2 & 95.81 / \textbf{0.99} / 93.23 & 97.27 / \textbf{1.42} / 95.21 \\
$N=3$, mean-loss, poly2 & 96.00 / 1.00 / 93.02 & 97.81 / 2.16 / 96.35 \\
\bottomrule
\end{tabular}
\end{table}

\end{document}